%% file: main.tex
\documentclass[sigconf]{acmart}

\copyrightyear{2026}
\acmYear{2026}
\setcopyright{cc}
\setcctype{by}
\acmConference[WSDM '26]{Proceedings of the Nineteenth ACM International Conference on Web Search and Data Mining}{February 22--26, 2026}{Boise, ID, USA}
\acmBooktitle{Proceedings of the Nineteenth ACM International Conference on Web Search and Data Mining (WSDM '26), February 22--26, 2026, Boise, ID, USA}
\acmPrice{}
\acmDOI{10.1145/3773966.3778000}
\acmISBN{979-8-4007-2292-9/2026/02}

\usepackage{algorithmic}
\usepackage{graphicx}
\usepackage{textcomp}
\usepackage{xcolor}
    
\usepackage{subfiles}
\usepackage{paralist}
\usepackage[mathscr]{eucal} 
\usepackage{algorithmic}

\usepackage[noend,ruled,vlined,linesnumbered,resetcount]{algorithm2e}
\usepackage{booktabs}
\usepackage{graphicx}
\usepackage[T1]{fontenc}
\usepackage{textcomp}
\usepackage{xcolor}
\usepackage{comment}
\usepackage{hyperref}
\usepackage{pifont}
\usepackage{makecell}
\usepackage{balance}

\usepackage{multirow}
\usepackage{tabularx,stackengine,collcell}
\let\endminwd\relax
\newcolumntype{L}[1]{>{\collectcell\xminwd l{#1}}l<{\endminwd\endcollectcell}}
\newcolumntype{C}[1]{>{\collectcell\xminwd c{#1}}c<{\endminwd\endcollectcell}}
\newcolumntype{R}[1]{>{\collectcell\xminwd r{#1}}r<{\endminwd\endcollectcell}}
\def\minwd#1#2#3\endminwd{\stackengine{0pt}{#3}{\rule{#2}{0pt}}{O}{#1}{F}{F}{L}}
\newcommand\xminwd[1]{\minwd#1}

\usepackage[center]{subfigure}
\usepackage{caption}
\usepackage{enumitem}
\usepackage{amsfonts}
\usepackage[table]{xcolor}  
\usepackage{bbm}

\newtheorem{theorem}{Theorem}

\SetAlFnt{\small}
\SetKwComment{Comment}{$\triangleright$\ }{}

\SetCommentSty{mycommfont}
\SetKwInput{KwData}{Input}
\SetKwInput{KwResult}{Output}

\begin{document}

\title{Sequential Data Augmentation for Generative Recommendation}

\settopmatter{authorsperrow=4}
\author{Geon Lee}
\authornote{Work completed during internship at Snap Inc.}
    \affiliation{%
        \institution{KAIST}
        \city{Seoul}
        \country{Republic of Korea}
    }
    \email{geonlee0325@kaist.ac.kr}
\author{Bhuvesh Kumar} 
    \affiliation{%
        \institution{Snap Inc.}
        \city{Bellevue}
        \state{WA}
        \country{USA}
    }
    \email{bkumar4@snap.com}
\author{Mingxuan Ju} 
    \affiliation{%
        \institution{Snap Inc.}
        \city{Bellevue}
        \state{WA}
        \country{USA}
    }
    \email{mju@snap.com}
\author{Tong Zhao} 
    \affiliation{%
        \institution{Snap Inc.}
        \city{Bellevue}
        \state{WA}
        \country{USA}
    }
    \email{tong@snap.com}
\author{Kijung Shin} 
    \affiliation{%
        \institution{KAIST}
        \city{Seoul}
        \country{Republic of Korea}
    }
    \email{kijungs@kaist.ac.kr}
\author{Neil Shah} 
    \affiliation{%
        \institution{Snap Inc.}
        \city{Bellevue}
        \state{WA}
        \country{USA}
    }
    \email{nshah@snap.com}
\author{Liam Collins} 
    \affiliation{%
        \institution{Snap Inc.}
        \city{Bellevue}
        \state{WA}
        \country{USA}
    }
    \email{lcollins2@snap.com}




\begin{abstract}
\input{000abstract.tex}
\end{abstract}

\input{dfn.tex}

\maketitle

\section{Introduction}
\label{sec:intro}
\input{010intro.tex}

\section{Preliminaries}
\label{sec:prelim}
\input{020prelim}

\section{Analysis of Data Augmentation Strategies}
\label{sec:prelim_study}
\input{021prelim}

\section{Analysis of Augmented Training Data}
\label{sec:analysis}
\input{031analysis}

\section{\method: A Generalized and Principled Framework for Data Augmentation}
\label{sec:unified}
\input{030analysis}





\section{Experimental Results}
\label{sec:exp}

\input{040exp}

\section{Related Work}
\label{sec:related}

\input{050related}

\section{Conclusions and Future Directions}
\label{sec:conc}
\input{060conc}

\section*{Acknowledgments}
This work was partly supported by Institute of Information \& Communications Technology Planning \& Evaluation (IITP) grant funded by the Korea government (MSIT) (No. RS-2024-00438638, EntireDB2AI: Foundations and Software for Comprehensive Deep Representation Learning and Prediction on Entire Relational Databases).

\appendix
\label{sec:apx}
\input{070appendix}

\section*{Ethical Considerations}
\label{sec:ethic}
\input{080ethic}

\bibliographystyle{ACM-Reference-Format}
\balance
\bibliography{ref}

\end{document}

%% file: 000abstract.tex
Generative recommendation plays a crucial role in personalized systems, predicting users’ future interactions from their historical behavior sequences.
A critical yet underexplored factor in training these models is data augmentation, the process of constructing training data from user interaction histories.
By shaping the training distribution, data augmentation directly and often substantially affects model generalization and performance.
Nevertheless, in much of the existing work, this process is simplified, applied inconsistently, or treated as a minor design choice, without a systematic and principled understanding of its effects.

Motivated by our empirical finding that different augmentation strategies can yield large performance disparities, we conduct an in-depth analysis of how they reshape training distributions and influence alignment with future targets and generalization to unseen inputs.
To systematize this design space, we propose \method, a generalized and principled framework that models augmentation as a stochastic sampling process over input–target pairs with three bias-controlled steps: sequence sampling, target sampling, and input sampling. 
This formulation unifies widely used strategies as special cases and enables flexible control of the resulting training distribution.
Our extensive experiments on benchmark and industrial datasets demonstrate that \method yields superior accuracy, data efficiency, and parameter efficiency compared to existing strategies, providing practical guidance for principled training data construction in generative recommendation.
Our code is available at \url{https://github.com/snap-research/GenPAS}.

%% file: dfn.tex
\newcommand\red[1]{\textcolor{red}{#1}}
\newcommand\blue[1]{\textcolor{blue}{#1}}

\definecolor{seagreenbright}{rgb}{0.25, 0.80, 0.45}
\newcommand\ourgreen[1]{\textcolor{seagreenbright}{#1}}

\newcommand{\smallsection}[1]{{\vspace{0.05in} \noindent {\bf{\underline{\smash{#1}}}}}}

\newtheorem{assumption}[theorem]{Assumption}

\newcommand{\cmark}{\ding{51}}  
\newcommand{\xmark}{\ding{55}}  

\newcommand{\method}{\textsc{GenPAS}\xspace}

%% file: 010intro.tex
Generative recommendation (GR) plays a crucial role in modern personalized systems by training models to predict the next item a user will interact with based on historical behavior sequences. 
These systems are widely deployed in real-world applications in which understanding a user's evolving interests is critical for delivering timely and relevant recommendations \cite{singh2024betteryoutube, deng2025onerec, Kim2025amazon, grid}.

A common setting in GR is to leverage past user interactions to train models for predicting future ones. 
A key challenge is data sparsity, where users engage with only a small fraction of available items, making it difficult for models to accurately predict future behavior. 
Thus, it is crucial to effectively exploit the available historical data to train models that generalize well and provide accurate predictions on unseen future interactions \cite{yang2024unifyingmeta}.

An often overlooked yet critical design choice in GR is how to curate training samples from user interaction histories.
This process, referred to as \textit{data augmentation}, is often simplified, applied inconsistently, or treated as a minor implementation detail.
However, data augmentation fundamentally dictates the statistical properties of the training data, which can substantially influence model generalization and overall performance~\cite{zhou2024contrastive}.
Nevertheless, much of the progress in GR has focused on architectural or optimization techniques~\cite{kang2018self,rajput2023recommender}, with limited attention paid to understanding how data augmentation impacts model performance.

To address this gap, we conduct an in-depth analysis of data augmentation in GR.
We examine three widely used strategies for constructing training data, specifically, Last-Target, Multi-Target, and Slide-Window.
Each strategy reshapes the training distribution in distinct ways, particularly in terms of (i) the target distribution (i.e., how frequently each item is selected as a prediction target) and (ii) the joint input-target distribution (i.e., how historical interaction sequences are paired with their prediction targets). 
Our empirical results show that the choice of augmentation strategy can substantially affect model performance, with observed differences of up to 783.7\% in NDCG@10.

Building on these observations, we investigate which properties of augmented training data make some strategies outperform others.
First, we empirically and theoretically show that stronger alignment between the training and future target distributions correlates with higher performance.
Second, we examine how training input composition affects generalization to unseen input sequences.
We capture these effects by introducing two measures: \textit{alignment}, which measures how well the training data includes inputs structurally similar to a given test input with the same target, and \textit{discrimination}, which measures how well it includes structurally dissimilar inputs with different targets.
We observe that model performance is strongly related to the balance between these measures.
This raises a natural question: \textit{Do widely used data augmentation strategies produce an optimal training distribution for achieving these properties, or are they merely suboptimal heuristics?}

To answer this question, we present \method, a generalized and principled augmentation framework.
\method interprets data augmentation as a stochastic sampling process over input–target pairs, decomposed into three fundamental steps: sequence sampling, target sampling, and input sampling. 
By adjusting its bias parameters at each step, \method can flexibly shape both the target distribution and the joint input–target distribution for each dataset.
Notably, common data augmentation strategies are special cases corresponding to particular parameter settings, leaving room for discovering more effective configurations.
Finally, we present a heuristic yet effective parameter search scheme that reduces the search space by filtering out configurations that, according to our earlier empirical analysis, exhibit poor training properties.

Our experiments across benchmark and industrial datasets demonstrate the effectiveness of \method.
When parameters are properly tuned for each dataset, GR models trained on \method-augmented data significantly outperform those trained with conventional strategies.
Furthermore, \method consistently surpasses common input sequence–level augmentations (e.g., insertion, deletion, reordering).
\method is also both data-efficient (i.e., achieving strong results with less training data) and parameter-efficient (i.e., enabling smaller models with \method to outperform large models with alternative approaches).
These results emphasize the critical role of a properly shaped training data distribution in driving both the generalization ability and the overall effectiveness of GR models. 

We summarize our contributions as follows:
\begin{itemize}[leftmargin=*]
    \item \textbf{In-Depth Analysis.} 
    We conduct a thorough empirical and theoretical analysis of widely used data augmentation strategies for GR.
    Our findings reveal that these strategies reshape the target and input–target distributions, which significantly influence model generalization and performance.

    \item \textbf{Generalized Framework.} 
    We introduce \method, a principled three-step sampling framework for constructing training data that unifies existing strategies as special cases. Its bias-control parameters allow precise control over the target and input–target distributions, enabling flexible adaptation to different datasets.
    
    \item \textbf{Strong Performance.} 
    Across both benchmark and industrial datasets, we show that GR models trained on data augmented by \method achieve substantial gains in accuracy and efficiency.
\end{itemize}

\noindent
Our code is available at \url{https://github.com/snap-research/GenPAS}.


%% file: 020prelim.tex
In this section, we review the problem of generative recommendation (GR) and common approaches to constructing training data.

\subsection{Problem Setup}
Let $\mathcal{U}$ and $\mathcal{I}$ denote the sets of users and items, respectively.  
Each user $u \in \mathcal{U}$ has an interaction history of an ordered sequence $s^{(u)} = [i_1^{(u)}, i_2^{(u)}, \dots, i_{|s^{(u)}|}^{(u)}]$, where $i_k^{(u)} \in \mathcal{I}$ denotes the $k$\textsuperscript{th} item interacted with by user $u$.  
The goal of GR is to predict the next item $i_{|s^{(u)}|+1}^{(u)} \in \mathcal{I}$ that the user will engage with, {based on its past interactions, i.e., $s^{(u)}$.}
Training GR models requires constructing a dataset of input–target pairs derived from user sequences.


\subsection{Training Data Construction}\label{sec:prelim:strategies}
While most works adopt a \textit{leave-last-out} data split, where the last item in each user sequence is held out for testing, the second-to-last item for validation, and the remaining preceding items for training~\cite{kang2018self,rajput2023recommender,xie2022contrastive,zhou2020s3}, the process of constructing the training data itself is far less standardized. 
Importantly, this construction process can be viewed as a form of \textit{data augmentation}, in which multiple input-target pairs are derived from each user sequence. 

Formally, the training data is the union of user-specific training sets, $\mathcal{D}_\text{train}=\bigcup_{u\in \mathcal{U}}\mathcal{D}_\text{train}^{(u)}$, where $\mathcal{D}_\text{train}^{(u)}$ is the set of input-target pairs constructed from user $u$'s interaction sequence $s^{(u)}$.
We review three common strategies for constructing $\mathcal{D}_\text{train}^{(u)}$ in GR.

\smallsection{Last-Target Strategy (LT).}
A widely-adopted approach is to generate a single training sample per user using the last item in the training sequence as the prediction target~\cite{zhou2024contrastive}. 
For each user $u\in\mathcal{U}$, the input is all but the last item, and the target is the last item:
\begin{equation*}
    x^{(u)}=\left[i_1^{(u)}, \dots, i_{|s^{(u)}|-1}^{(u)}\right]\;,\;\;\;\; y^{(u)}=i_{|s^{(u)}|}^{(u)}.
\end{equation*}
The resulting training set for user $u$ is $\mathcal{D}_\text{train}^{(u)}=\left\{(x^{(u)},y^{(u)})\right\}$.

\smallsection{Multi-Target Strategy (MT).}
Some methods implicitly~\cite{kang2018self} or explicitly~\cite{zhou2020s3,chen2022intent} predict multiple targets within each user sequence. 
For each position $k\in\{2,\dots,|s^{(u)}|\}$, the first $k-1$ items are used as the input and the $k$\textsuperscript{th} item as the target:
\begin{equation*}
    x_k^{(u)}=\left[i_1^{(u)}, \dots, i_{k-1}^{(u)}\right]\;,\;\;\;\; y_k^{(u)}=i_{k}^{(u)}.
\end{equation*}
This produces $|s^{(u)}|-1$ input-target pairs for each user sequence:
\begin{equation*}
\mathcal{D}_\text{train}^{(u)} = \left\{ (x_k^{(u)}, y_k^{(u)}) \mid 2\leq k\leq|s^{(u)}| \right\}.
\end{equation*}
This strategy subsumes the last-target strategy, while providing additional supervision signals from earlier positions in the sequence.

\smallsection{Slide-Window Strategy (SW).}
Some methods, such as TIGER~\cite{rajput2023recommender}, apply a more aggressive slide-window augmentation that extracts all possible contiguous subsequences from each user sequence.~\footnote{While the original paper~\cite{rajput2023recommender} did not specify the exact construction procedure, a recent reproducibility study~\cite{ju2025generative} reported that all possible window sizes were used.}
For each target position $k \in \{2, \dots, |s^{(u)}|\}$, all contiguous subsequences ending at position $k-1$ are considered as inputs.
Formally, for each window size $w \in \{1, \dots, k-1\}$, a training pair is defined:
\begin{equation*}
x_{k,w}^{(u)} = \left[i_{k-w}^{(u)}, \dots, i_{k-1}^{(u)}\right], \quad y_{k,w}^{(u)} = i_{k}^{(u)}.
\end{equation*}
This produces $|s^{(u)}|(|s^{(u)}|-1)/2$ input-target pairs for each user sequence. 
The resulting training set for user $u$ is:
\begin{equation*}
\mathcal{D}_\text{train}^{(u)} = \left\{ (x_{k,w}^{(u)}, y_{k,w}^{(u)}) \mid 2 \le k \le |s^{(u)}|,\; 1 \le w \le k-1 \right\}.
\end{equation*}
This strategy subsumes the causal multi-target strategy, while introducing more diverse inputs associated with each target.


%% file: 021prelim.tex

In this section, we empirically examine how existing data augmentation strategies, discussed in Section~\ref{sec:prelim:strategies}, influence training data distribution and, in turn, the performance of GR models.

\subsection{Impact on Training Data Distribution}\label{sec:prelim:distribution}
We examine how different augmentation strategies reshape the training data distribution, which consequently affects model performance, as shown in the next subsections.

\smallsection{Impact on Target Distribution.}
In Figure~\ref{fig:target_distribution}, we compare the target probability, i.e., normalized frequency with which each item appears as the prediction target.
The distributions are clearly distinct.
LT yields training data in which a relatively large fraction of items have high target probabilities, leading to a more skewed distribution.
In contrast, MT and SW include more items with lower target probabilities, although their overall distribution shapes differ.

\smallsection{Impact on Input-Target Pair Distribution.}
In Figure~\ref{fig:input_distribution}, we compare the number of inputs associated with each target under different strategies. 
The distributions differ substantially.
LT links most targets to only a few inputs, while MT increases the number of inputs per target by using more positions in the sequence as targets, and SW yields the most inputs per target by enumerating all possible subsequences.

\begin{figure}[t]
    \centering
    \includegraphics[width=0.63\linewidth]{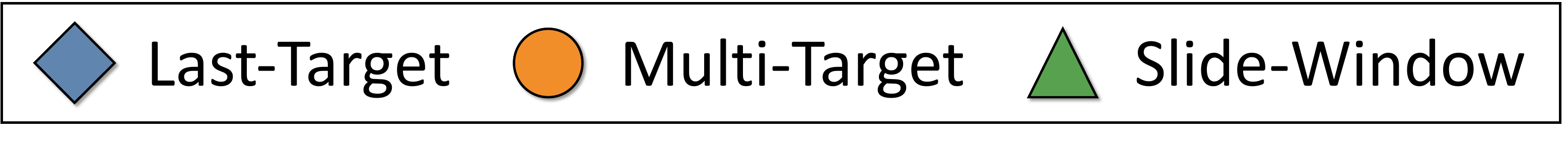}\\
    \includegraphics[width=0.46\linewidth]{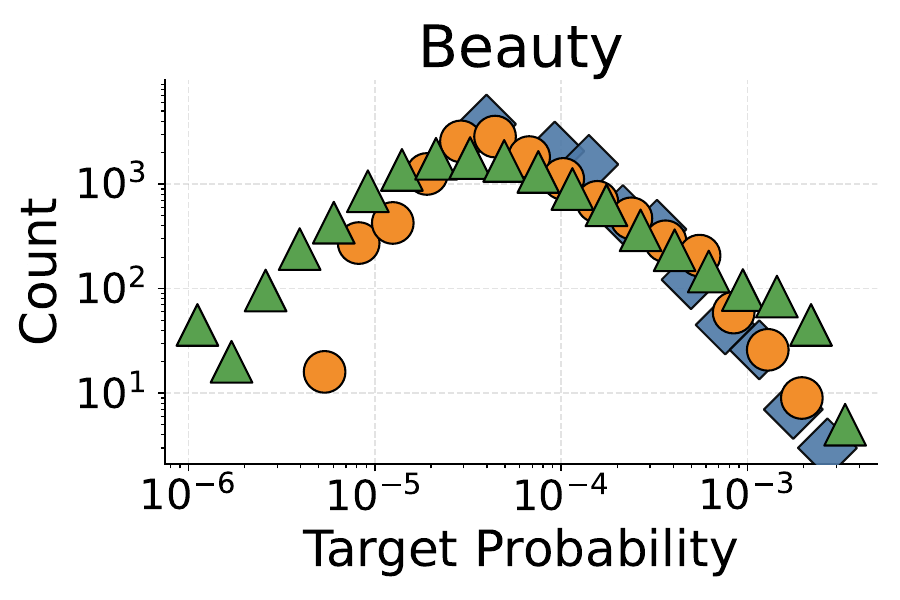}
    \hspace{8pt}
    \includegraphics[width=0.46\linewidth]{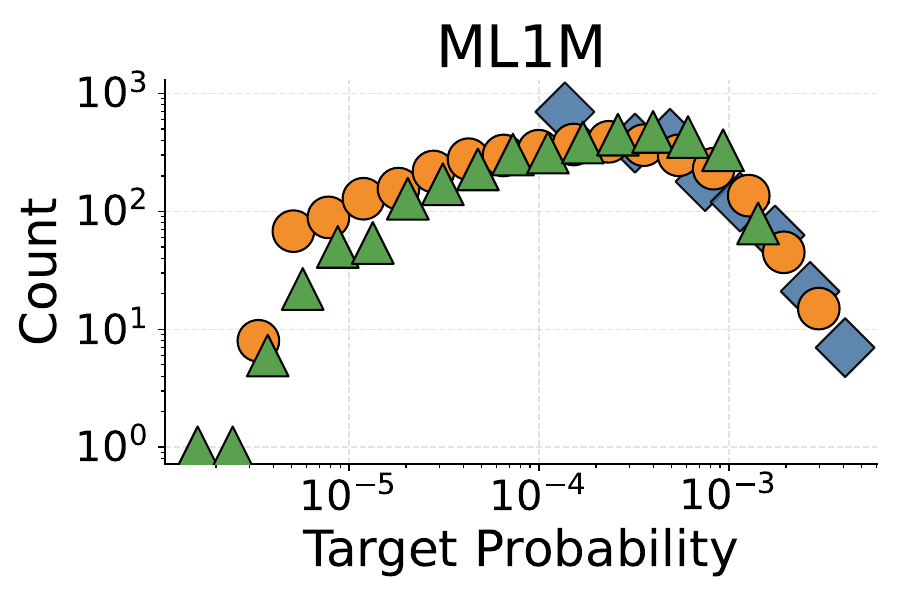}
    \caption{
    Different strategies yield distinct target distributions.
    LT skews toward 
    {frequent items}, while MT and SW produce a more balanced distribution.
    \label{fig:target_distribution}}
\end{figure}

\begin{figure}[t]
    \centering
    \includegraphics[width=0.63\linewidth]{FIG_NEW/legend_strategy.pdf}\\
    \includegraphics[width=0.46\linewidth]{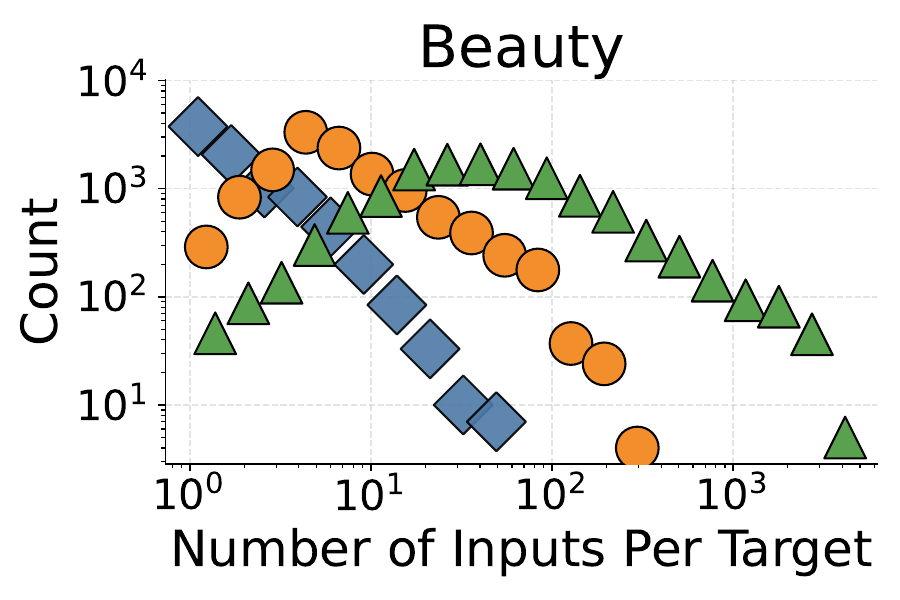}
    \hspace{8pt}
    \includegraphics[width=0.46\linewidth]{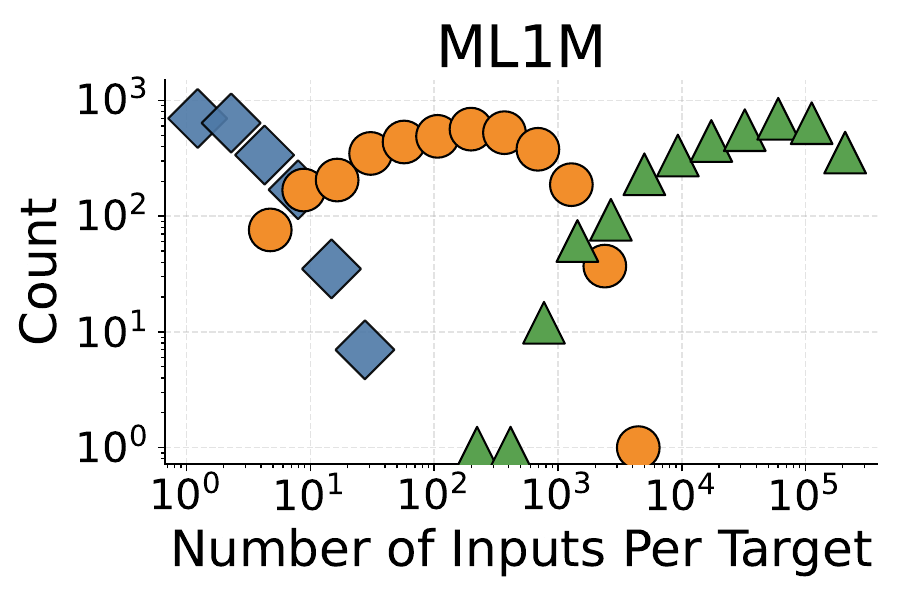}
    \caption{
    Different strategies produce distinct input-target distributions. 
    LT yields few inputs per target, MT increases this with more target positions, and SW produces the most by enumerating all subsequences.
    \label{fig:input_distribution}}
\end{figure}

\subsection{Impact on Model Performance}
We evaluate SASRec~\cite{kang2018self}~\footnote{We implement SASRec without its
MT objective.} and TIGER~\cite{rajput2023recommender} under different data augmentation strategies.
As shown in Table~\ref{tab:prelim}, both models exhibit substantial performance variation depending on the strategy.
For instance, TIGER achieves a 783.7\% higher NDCG@10 on ML1M with MT compared to LT.
These results demonstrate that \textit{GR models are highly sensitive to training data distribution}, emphasizing the critical role of data augmentation in model effectiveness.

Interestingly, although SW generates the most diverse set of input-target pairs, it often underperforms MT, which produces a strict subset of these pairs.
This indicates that \textit{the shape of the training distribution can play a more critical role than the simple diversity of training samples}.

In addition, performance variation resulting from different data augmentation strategies often exceeds that from changing model architectures.
For example, under the same setting, the maximum performance gap across models is 305.4\% (LT on Sports), while across strategies it reaches 783.7\%.
This shows that \textit{the choice of augmentation strategy can have a greater impact than architecture changes}. 
While prior work has mainly focused on architectural improvements, our findings suggest that shaping training data can be even more influential for final performance.

%% file: 031analysis.tex
We systematically analyze how data augmentation strategies reshape the training distribution and affect the performance of GR models. 
Specifically, we examine their effects on the target distribution and the joint input–target distribution, and how the reshaped training distributions contribute to model performance.

\begin{table}[t!]
    \centering
    \caption{
    Performance (w.r.t. NDCG@10) of GR models varies greatly across augmentation strategies, showing that model effectiveness is highly sensitive to the chosen strategy, with large gaps between best and worst cases.
    }
    \label{tab:prelim}
    \setlength\tabcolsep{4.1pt}
    \renewcommand{\arraystretch}{1.1}
    \scalebox{0.87}{
    \begin{tabular}{ll|ccccc}
        \toprule
        Model & Strategy & \textbf{Beauty} & \textbf{Toys} & \textbf{Sports} & \textbf{ML1M} & \textbf{ML20M}\\
        \midrule
        \multirow{4}{*}{SASRec} & Last-Target  & 0.0124 & 0.0121 & 0.0037 & 0.0136 & \underline{\smash{0.0628}} \\
        & Multi-Target  & \textbf{0.0372} & \textbf{0.0378} & \textbf{0.0162} & \textbf{0.1194} & \textbf{0.0995} \\
        & Slide-Window & \underline{\smash{0.0323}} & \underline{\smash{0.0354}} & \underline{\smash{0.0149}} & \underline{\smash{0.1022}} & 0.0526 \\
        \rowcolor{gray!10}
        \cellcolor{white} & $\Delta$ (Best/Worst) & 200.0\% & 212.4\% & 337.8\% & 777.9\% & 89.2\%\\
        \midrule
        \multirow{4}{*}{TIGER} & Last-Target  & 0.0213 & 0.0212 & 0.0150 & 0.0147 & \underline{\smash{0.0559}} \\
        & Multi-Target  & \underline{\smash{0.0319}} & \textbf{0.0303} & \textbf{0.0194} & \textbf{0.1299} & \textbf{0.1147} \\
        & Slide-Window & \textbf{0.0321} & \underline{\smash{0.0273}} & \underline{\smash{0.0171}} & \underline{\smash{0.1105}} & 0.0321 \\
        \rowcolor{gray!10}
        \cellcolor{white} & $\Delta$ (Best/Worst) & 50.7\% & 42.9\% & 29.3\% & 783.7\% & 257.3\%\\
        \bottomrule
    \end{tabular}}
\end{table}

\subsection{Analysis of Target Distribution}
We first analyze how data augmentation strategies shape the target distribution in the training data, which plays a crucial role in model training.
The proportion of times an item $y \in \mathcal{I}$ appears as a training target is
$p_\text{train}(y) = \frac{1}{|\mathcal{D}_\text{train}|} \sum_{(x',y')\in\mathcal{D}_\text{train}} \mathbbm{1}[y=y']$ where $\mathbbm{1}[\cdot]$ is the indicator function.



\smallsection{Empirical Observations.}
For effective generalization, the training target distribution should closely align with that at test time.
As observed in Section~\ref{sec:prelim:distribution}, each strategy induces a distinct target distribution.
We quantify these differences via KL divergence between the training and test target distributions, i.e., $\text{KL}(p_\text{train} \mid\mid p_\text{test}) = \sum_{y\in \mathcal{I}} p_\text{train}(y) \log\frac{p_\text{train}(y)}{p_\text{test}(y)}$.
As shown in Table~\ref{tab:kl_align_disc}, MT consistently yields the lowest KL divergence across datasets, while LT has the highest in most cases, indicating substantial misalignment.
This trend mirrors the performance results in Table~\ref{tab:prelim}, where MT outperforms the other two strategies.
Notably, on ML20M, LT not only uniquely outperforms SW but also achieves lower KL divergence.
This suggests that closer alignment between training and test target distributions correlates with better performance.


\begin{table*}[t!]
    \centering
    \caption{
        Training distributions (specifically, target distributions and input-target distributions) can affect model performance. 
        We report KL divergence (KL) between training and test target distributions, alignment (A), discrimination (D) of input-target distributions in both sets, and their ratio (A/D) across datasets and strategies. 
        For readability, alignment (A) and discrimination (D) values are scaled by a factor of 100.
        The most desirable value is in \textbf{bold}, and the second-most desirable value is \underline{\smash{underlined}}.
    }
    \label{tab:kl_align_disc}
    \setlength\tabcolsep{2.8pt}
    \renewcommand{\arraystretch}{1.1}
    \scalebox{0.85}{
    \begin{tabular}{l|>{\columncolor{gray!10}}c|cc>{\columncolor{gray!10}}c|>{\columncolor{gray!10}}c|cc>{\columncolor{gray!10}}c|>{\columncolor{gray!10}}c|cc>{\columncolor{gray!10}}c|>{\columncolor{gray!10}}c|cc>{\columncolor{gray!10}}c|>{\columncolor{gray!10}}c|cc>{\columncolor{gray!10}}c}
        \toprule
        & \multicolumn{4}{c|}{\textbf{Beauty}} 
        & \multicolumn{4}{c|}{\textbf{Toys}} 
        & \multicolumn{4}{c|}{\textbf{Sports}} 
        & \multicolumn{4}{c|}{\textbf{ML1M}} 
        & \multicolumn{4}{c}{\textbf{ML20M}} \\
        \cmidrule(lr){2-5}\cmidrule(lr){6-9}\cmidrule(lr){10-13}\cmidrule(lr){14-17}\cmidrule(lr){18-21}
        \textbf{Strategy} & KL {\footnotesize($\downarrow$)} & A {\footnotesize($\uparrow$)} & D {\footnotesize($\downarrow$)} & A/D {\footnotesize($\uparrow$)} & KL {\footnotesize($\downarrow$)} & A {\footnotesize($\uparrow$)} & D {\footnotesize($\downarrow$)} & A/D {\footnotesize($\uparrow$)} & KL {\footnotesize($\downarrow$)} & A {\footnotesize($\uparrow$)} & D {\footnotesize($\downarrow$)} & A/D {\footnotesize($\uparrow$)} & KL {\footnotesize($\downarrow$)} & A {\footnotesize($\uparrow$)} & D {\footnotesize($\downarrow$)} & A/D {\footnotesize($\uparrow$)} & KL {\footnotesize($\downarrow$)} & A {\footnotesize($\uparrow$)} & D {\footnotesize($\downarrow$)} & A/D {\footnotesize($\uparrow$)} \\
        \midrule
        Last-Target 
        & 2.768 & 0.669 & 0.046 & 14.54
        & 3.147 & 0.765 & 0.039 & 19.61
        & 2.737 & 0.307 & 0.036 & 8.52
        & 2.198 & 0.454 & 0.116 & \underline{\smash{3.91}}
        & 0.343 & 0.534 & 0.065 & \underline{\smash{8.21}} \\
        
        Multi-Target 
        & \textbf{0.898} & 0.788 & 0.048 & \textbf{16.42}
        & \textbf{1.062} & 0.845 & 0.042 & \textbf{20.11}
        & \textbf{0.819} & 0.373 & 0.041 & \textbf{9.09}
        & \textbf{0.495} & 0.450 & 0.104 & \textbf{4.32}
        & \textbf{0.168} & 0.447 & 0.035 & \textbf{12.77} \\
        
        Slide-Window 
        & \underline{1.158} & 0.654 & 0.044 & \underline{\smash{14.86}}
        & \underline{1.349} & 0.716 & 0.036 & \underline{\smash{19.88}}
        & \underline{0.910} & 0.334 & 0.038 & \underline{\smash{8.78}}
        & \underline{0.563} & 0.632 & 0.177 & 3.57
        & \underline{\smash{0.312}} & 0.339 & 0.046 & 7.37 \\
        \bottomrule
    \end{tabular}}
\end{table*}

\smallsection{Theoretical Analysis.}
We next  quantify how the discrepancy between training and test target distributions depends on the underlying data generative process and the augmentation strategy in a controlled setting.
We start by assuming that user sequences are sampled independently from a ground-truth population distribution  $P_{\text{pop}}$, and all supported sequences have the length $n$. 
\begin{assumption}[I.I.D. Users]\label{assump:indep_users} 
(Sequence, target) tuples are sampled independently from a distribution $P_{\text{pop}}$, i.e. $(s^{(u)},i^{(u)}_{|s^{(u)}|+1}) \stackrel{\text{i.i.d.}}{\sim} P_{\text{pop}}$. Further, $P_{\text{pop}}(s^{(u)},i_{|s|+1}^{(u)}) > 0  \implies |s^{(u)}|=n$.
\end{assumption}
Furthermore, since in this section we are not concerned with the relationship between inputs and targets, we consider that all items are sampled independently such that the only factor that influences their distribution is their position in the sequence.
\begin{assumption}[Independent Items]\label{assump:indep_items} 
User sequences and targets are constructed by independently sampling $n+1$ items, that is: 
for all  $(s^{(u)},i^{(u)}_{n+1})$ s.t. $|s^{(u)}|=n$,
 $P_{\text{pop}}(s^{(u)}, i^{(u)}_{n+1}) = \prod_{k=1}^{n+1} p_{k}(i^{(u)}_{k})$, where $p_k$ is a distribution over $\mathcal{I}$ for all $k \in \{1,\dots, n+1\}$. 
\end{assumption}
Note that $p_{k}$ is the item distribution at position $k$ in the sequences. We can interpret the relationship between the $p_{k}$'s as dictating the recency bias in the data. 
{The greater the distance between $p_k$ and $p_{n+1}$ for smaller $k$, the stronger the recency bias.}
Here we use Total Variation (TV)
to measure the distance between distributions:
    $\text{TV}(p,q) := \frac{1}{2}\sum_{y\in \mathcal{I}} |p(y) - q(y)|$.
Moreover, we denote the empirical training target distributions induced by the LT, MT, and SW augmentations as $p_{\text{train}}^{\text{LT}}$, $p_{\text{train}}^{\text{MT}}$, and $p_{\text{train}}^{\text{SW}}$, respectively. We are interested in the distances between these empirical training target distributions and the ground-truth test target distribution $p_{n+1}$, which the following result bounds (see Appendix~\ref{sec:theory} and \cite{online2025appendix}).
\begin{theorem} \label{thm:targets}
    
        Suppose Assumptions \ref{assump:indep_users} and \ref{assump:indep_items} hold. Denote $\delta_k := \text{TV}(p_k, p_{n+1})$.  Then, with probability at least $0.99$,
    {\small\begin{align*}
        TV\left(p_{\text{train}}^{\text{LT}}, p_{n+1}\right) &= 
        O\left({ \delta_n + |\mathcal{I}|\sqrt{\frac{\log(|\mathcal{I}|)}{|\mathcal{U}|} }}\right), \\
        TV\left(p_{\text{train}}^{\text{MT}}, p_{n+1}\right) &= 
        O\left({\frac{1}{n}\sum_{k=2}^n \delta_k + |\mathcal{I}|\sqrt{\frac{\log(|\mathcal{I}|)}{n|\mathcal{U}|} }}\right), \\
        TV\left(p_{\text{train}}^{\text{SW}}, p_{n+1}\right) &= 
        O\left({\frac{1}{n^2}\sum_{k=2}^n k \delta_k + |\mathcal{I}|\sqrt{\frac{\log(|\mathcal{I}|) }{n|\mathcal{U}|} }}\right).
    \end{align*}}
\end{theorem}
The bounds in Theorem \ref{thm:targets} consist of two terms: (1) a distribution bias that scales with the $\delta_k$'s, and (2) a sampling variance that diminishes with more target samples. LT relies only on the last item per user for training targets, so its bias scales only with $\delta_n$ and its variance diminishes only with the number of users.
On the other hand, MT and SW use within-sequence targets, so their bias grows with all $\delta_k\in \{2,\dots, n\}$ and their variance diminishes with the number of users {\em and} the sequence length. SW up-weights $\delta_k$'s with larger $k$ since it tends to sample later targets more often.  

These bounds suggest how to select  the  augmentation strategy depending on the underlying data distributions. If the data has high recency bias, i.e. $\delta_n \ll \dots \ll \delta_2$, and the bias term dominates the variance term, we should use LT. Conversely, if the data has low bias (small $\delta_k$'s) and the variance dominates, MT and SW induce a training target distribution closer to the test. 
This matches intuition: we should only train on sampled subsequences if the in-sequence targets are similarly distributed as the test targets.
In Appendix \ref{sec:theory}, we present a more general upper bound that captures a continuous range of augmentation strategies. We leave lower bounds (for which stronger assumptions are needed) to future work.

    


%



\subsection{Analysis of Input-Target Distribution}
We now analyze how each strategy shapes the distribution of input–target pairs in the training set and how this influences model performance.
Intuitively, given an unseen test input $x'$ with ground-truth target $y'$, the model is more likely to make accurate predictions if the training set contains a \textit{positive} input $x^+$ that is structurally \textit{similar} to $x'$ and shares the same target, i.e., $(x^+, y^+) \in \mathcal{D}_\text{train}$ with $y^+ = y'$.
Conversely, it is beneficial for the training set to also include \textit{negative} inputs $x^-$ that are structurally \textit{dissimilar} to $x'$ and have different targets, i.e., $(x^-, y^-) \in \mathcal{D}_\text{train}$ with $y^- \neq y'$.

\smallsection{Alignment and Discrimination.}
To quantify how well the training distribution supports generalization to unseen inputs, we introduce two measures.
Let $\mathcal{C}_y$ denote the set of inputs in the training set that are associated with target $y$, i.e., $\mathcal{C}_y = \{x \mid (x,y)\in \mathcal{D}_\text{train}\}$.

\textbf{Alignment} measures the expected similarity between an unseen input $x'$ and positive training inputs that share its target $y'$, averaged over input–target pairs from the test set:
\begin{equation}\label{eq:alignment}
    \mathbbm{E}_{(x',y')\sim\mathcal{D}_\text{test}} \biggr[ \frac{1}{|\mathcal{C}_{y'}|}\sum_{x^+ \in\mathcal{C}_{y'}} \textbf{sim}(x^+,x') \biggr],
\end{equation}
where $\textbf{sim}(\cdot,\cdot)$ denotes the similarity between two item sequences.~\footnote{We compute the similarity between two sequences as 1 minus the normalized edit (Levenshtein) distance~\cite{lcvenshtcin1966binary,yujian2007normalized}. See \cite{online2025appendix} for the details.}
Intuitively, a higher alignment score indicates that the training set contains inputs structurally similar to the test input with the correct target, increasing the likelihood of accurate predictions.

\textbf{Discrimination} measures the expected similarity between an unseen input $x'$ and negative training inputs with targets $y^- \neq y'$, averaged over test input–target pairs:
\begin{equation}\label{eq:discrimination}
    \mathbbm{E}_{(x',y')\sim\mathcal{D}_\text{test}} \biggr[ \mathbbm{E}_{y^-\in \mathcal{Y}\setminus \{y'\}} \biggr[\frac{1}{|\mathcal{C}_{y^-}|}\sum_{x^- \in\mathcal{C}_{y^-}}  \textbf{sim}(x^-,x') \biggr]\biggr].
\end{equation}
Intuitively, a lower discrimination score indicates that negative training inputs are less similar to the test input, reducing the chance of the model confusing it with {inputs associated with other targets.}

\smallsection{Empirical Observations.}
In Table~\ref{tab:kl_align_disc}, we report the alignment and discrimination scores of the training distributions induced by each strategy across datasets.
To assess their joint effect, we use the alignment-to-discrimination ratio, where a higher ratio indicates that the training set provides greater exposure to structurally similar positives while limiting confusing negatives.
We observe that MT consistently achieves the highest ratio, corresponding to its superior performance in Table~\ref{tab:prelim}.
This indicates that a good trade-off between alignment and discrimination is potentially important for improving generalization in GR models.

%% file: 030analysis.tex
Despite substantial progress in GR, there remains a lack of clarity and standardization in constructing training data from user interaction histories. 
Many existing works adopt different strategies discussed in Section~\ref{sec:prelim:strategies}, which leads to inconsistent training distributions. 
However, none of these strategies can be assumed to be best-suited for training GR model.
Thus, evaluations are often conducted under inconsistent or restrictive training distributions, making it difficult to accurately assess and compare the effectiveness of models.

This motivates the need for a \textit{generalized} and \textit{principled} framework for constructing training data for GR.
We present \method, (\textbf{\underline{\smash{Gen}}}eralized and \textbf{\underline{\smash{P}}}rincipled \textbf{\underline{\smash{A}}}ugmentation for \textbf{\underline{\smash{S}}}equences) which unifies existing strategies into a three-step sampling process. 
Further, this framework offers an efficient tuning strategy for more intentionally choosing the augmentation strategy.

\subsection{Principled Sampling Formulation}\label{sec:analysis:unified}
We find that input augmentations that sample non-contiguous subsequences (e.g., item insertion or deletion) do not significantly improve performance when combined with strategies that produce new targets beyond the last item in the sequence (see Table~\ref{tab:input_aug}). Therefore, 
following the strategies in Section~\ref{sec:prelim:strategies}, we consider that the input is a contiguous subsequence of a user's interaction history, and the target is the single item that immediately follows this subsequence.  
Let $\mathcal{X}$ denote the space of all such contiguous input subsequences,~\footnote{Formally, $\mathcal{X}=\{[i_p^{(u)},\dots,i_q^{(u)}]\mid u\in\mathcal{U},1\leq p \leq q \leq |s^{(u)}|\}$.} and let $\mathcal{Y}=\mathcal{I}$ denote the set of all items.

We formalize data augmentation in GR as a stochastic process that samples input-target pairs $(\tilde{x},\tilde{y})\in\mathcal{X}\times \mathcal{Y}$ from user interaction histories to construct the training set $\mathcal{D}_\text{train}$.
A specific augmentation strategy is characterized by a sampling distribution $p\left(\tilde{x},\tilde{y}\right)$, which defines the probability of selecting each input-target pair and thus determines the training distribution.

Building on this formulation, we introduce \method, a unified framework that generalizes the previously discussed augmentation strategies as a composition of three probabilistic steps:

\begin{enumerate}[leftmargin=*]
    \item \textbf{Sequence Sampling.}
    Sample a user $u \in \mathcal{U}$ with probability:
    \begin{equation}\label{eq:seq_sampling}
    p_\alpha(u) = \frac{(|s^{(u)}| - 1)^{\alpha}}{\sum_{u' \in \mathcal{U}} (|s^{(u')}| - 1)^{\alpha}},
    \end{equation}
    where $|s^{(u)}|$ is the length of $u$’s interaction sequence, and $|s^{(u)}| - 1$ corresponds to the number of valid target positions.
    The exponent $\alpha$ controls the bias toward users with different sequence lengths: $\alpha > 0$ favors longer sequences, $\alpha = 0$ samples all users uniformly, and $\alpha < 0$ favors shorter sequences.

    \item \textbf{Target Sampling.}  
    Given user $u$, select a target position $k \in \{2, \dots, |s^{(u)}|\}$ with probability:  
    \begin{equation}\label{eq:target_sampling}
        p_\beta(k \mid u) = \frac{(k - 1)^{\beta}}{\sum_{k'=2}^{|s^{(u)}|} (k' - 1)^{\beta}},
    \end{equation}
    and define the target item as $\tilde{y} = i_{k}^{(u)}$.
    The exponent $\beta$ controls the bias toward target recency: $\beta >0$ favors more recent targets, $\beta=0$ samples uniformly, and $\beta<0$ favors earlier targets.

    \item \textbf{Input Sampling.}  
    Given user $u$ and the chosen target position $k$, sample a start position $j \in \{1, \dots, k-1\}$ with probability:
    \begin{equation}\label{eq:input_sampling}
        p_\gamma(j \mid k, u) = \frac{j^{\gamma}}{\sum_{j'=1}^{k-1} (j')^{\gamma}},
    \end{equation}
    and define the input sequence as $\tilde{x} = [i_{j}^{(u)}, \dots, i_{k-1}^{(u)}]$.
    The exponent $\gamma$ controls the bias toward input sequence length (and thus how far back the context starts): $\gamma >0$ favors shorter contexts, $\gamma=0$ samples uniformly, and $\gamma<0$ favors longer contexts.
\end{enumerate}
\vspace{1mm}

\noindent
Given $\tilde{y} = i_{k}^{(u)}$ and $\tilde{x} = [i_{j}^{(u)}, \dots, i_{k-1}^{(u)}]$, each input–target pair $(\tilde{x}, \tilde{y}) \in \mathcal{X} \times \mathcal{Y}$ is uniquely determined by a tuple $(u, k, j)$,
where $u \in \mathcal{U}$ is the user, $k \in \{2, \dots, |s^{(u)}|\}$ is the target position, and $j \in \{1, \dots, k-1\}$ is the input start position.
The joint distribution over $(\tilde{x}, \tilde{y})$ therefore factorizes as:
\begin{equation*}
p(\tilde{x}, \tilde{y})
= p_\alpha(u) \cdot p_\beta(k \mid u) \cdot p_\gamma(j \mid k, u),
\end{equation*}
where $p_\alpha(u)$, $p_\beta(k \mid u)$, and $p_\gamma(j \mid k, u)$ are defined in Eqs.~\eqref{eq:seq_sampling} - \eqref{eq:input_sampling}.

\subsection{Recasting Existing Strategies}\label{sec:method:recast}
By appropriately setting $(\alpha, \beta, \gamma)$ in Eqs.~\eqref{eq:seq_sampling} – \eqref{eq:input_sampling}, we can recover common augmentation strategies as special cases of \method.
Table~\ref{tab:alpha_beta_gamma} summarizes the corresponding exponent choices:

\begin{itemize}[leftmargin=*]
    \item \textbf{Last-Target} ($\alpha = 0.0$, $\beta = \infty$, $\gamma = -\infty$):
    All users are sampled uniformly regardless of sequence length.
    For each user, the last item is always selected as the target, and the input sequence is fixed to the full prefix immediately preceding the target.
    \item \textbf{Multi-Target} ($\alpha = 1.0$, $\beta = 0.0$, $\gamma = -\infty$):
    Users are sampled in proportion to their sequence length.
    All valid target positions within a user’s sequence are selected uniformly, with the input again fixed to the full prefix before the target.
    \item \textbf{Slide-Window} ($\alpha = 2.0$, $\beta = 1.0$, $\gamma = 0.0$):
    Users with longer sequences are more likely to be sampled.
    Targets are chosen with a linear preference toward later positions, and the input start position is selected uniformly among the valid preceding indices, effectively sliding a variable-length window along the sequence.
\end{itemize}

\begin{table}[t!]
    \centering
    \caption{
        Common augmentation strategies (Section~\ref{sec:prelim:strategies}) can be viewed as special cases of $(\alpha,\beta,\gamma)$ in \method.
    }
    \vspace{-2pt}
    \label{tab:alpha_beta_gamma}
    \begin{minipage}{0.40\linewidth}
        \centering
        \setlength\tabcolsep{4.6pt}
        \renewcommand{\arraystretch}{1.1}
        \scalebox{0.87}{
        \begin{tabular}{l|ccc}
            \toprule
            Strategy & $\alpha$ & $\beta$ & $\gamma$\\
            \midrule
            Last-Target  & 0.0 & $\infty$ & $-\infty$ \\
            Multi-Target & 1.0 & 0.0 & $-\infty$\\
            Slide-Window & 2.0 & 1.0 & 0.0\\
            \bottomrule
        \end{tabular}}
    \end{minipage}
    \hfill
    \begin{minipage}{0.54\linewidth}
        \centering
        \includegraphics[width=\linewidth]{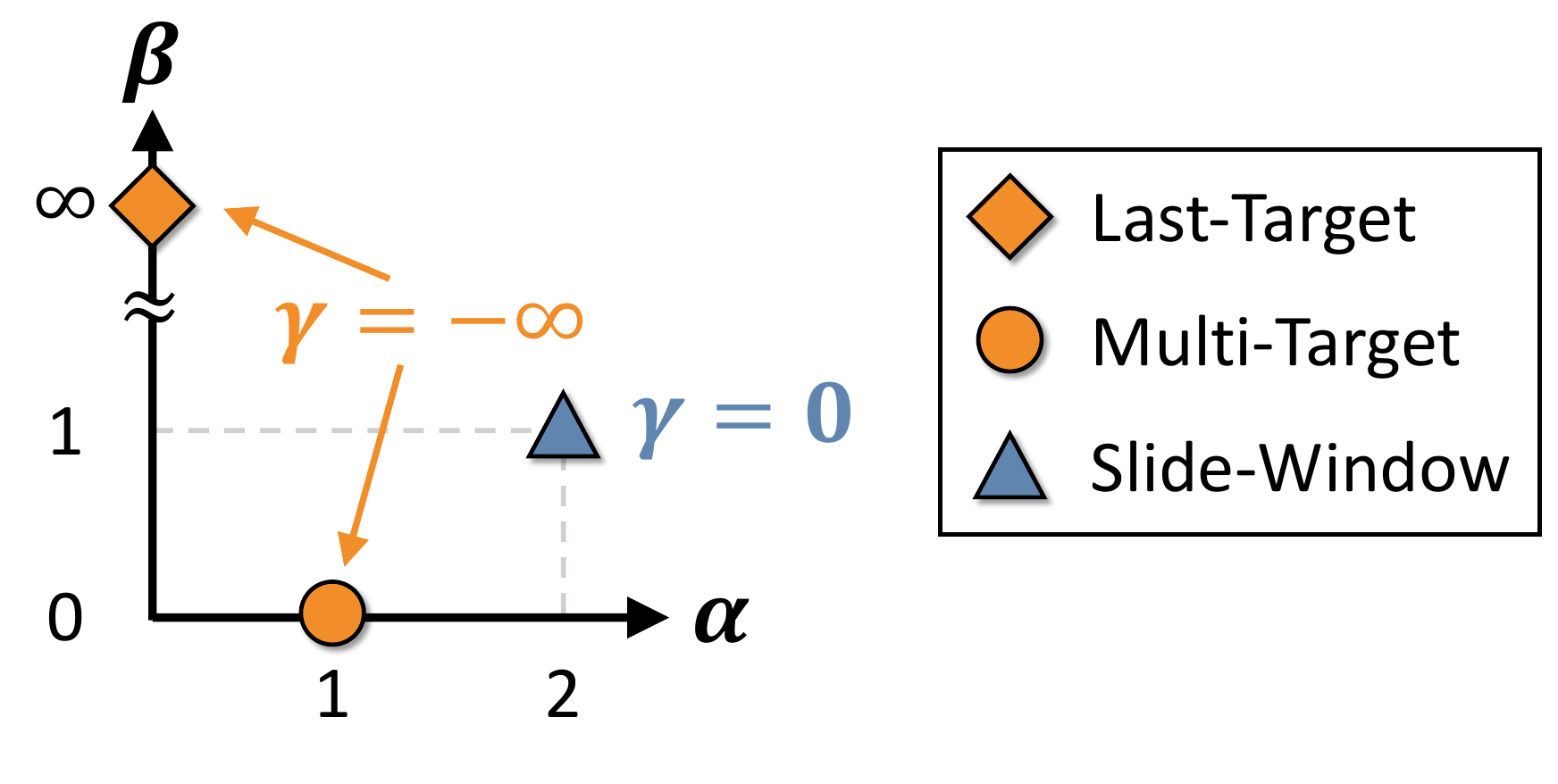}
    \end{minipage}
\end{table}

\noindent 
Such a parameterization enables systematic exploration of the training distribution space and facilitates principled selection of augmentation strategies beyond ad-hoc choices. 
See \cite{online2025appendix} for details.

\subsection{Reducing the Parameter Search Space}\label{sec:method:parameter_search}
While the hyperparameters $(\alpha,\beta,\gamma)$ can take infinitely many configurations, in practice, we reduce the search to a small subset that can still deliver effective performance 
under computational constraints.


To this end, we build on the empirical analysis in Section~\ref{sec:analysis}, where we observed that model performance is strongly linked to (O1) the alignment between the training and test target distributions, and (O2) the trade-off between alignment and discrimination in the input–target distribution, both of which depend upon the choice of $(\alpha, \beta, \gamma)$.  
Based on these insights, given a large set $\mathcal{S}$ of $(\alpha,\beta,\gamma)$ configurations, we design a two-step filtering procedure:  

\begin{enumerate}[leftmargin=*]
    \item \textbf{Target alignment filtering.}  
    Select the top $r$\% of configurations in $\mathcal{S}$ that induce the lowest KL divergence between the training and validation target distributions, motivated by (O1). 
    
    \item \textbf{Input–target trade-off filtering.}  
    From this subset, rank each configuration by $\max(\text{rank}_A,\text{rank}_D)$,  
    where $\text{rank}_A$ is the rank of alignment  (Eq.~\eqref{eq:alignment}; higher is better) and $\text{rank}_D$ is the rank of discrimination (Eq.~\eqref{eq:discrimination}; lower is better) computed using validation inputs and targets, or a sample thereof for faster processing.  
    Motivated by (O2), we take the maximum {of the two ranks} to reward configurations with strong alignment {\em and} discrimination.
    Then, we select the top-$k$ configurations for final training. 
\end{enumerate}
This procedure efficiently narrows the search space from hundreds to just a few configurations without training.
It is orthogonal to existing hyperparameter tuning methods~\cite{akiba2019optuna,bergstra2013making}, leveraging augmentation-specific signals (target KL divergence, alignment, and discrimination) that could also be integrated into those methods.


%% file: 040exp.tex
We conduct experiments to answer the following key questions:

\begin{itemize}[leftmargin=*]
    \item \textbf{Q1.} How does the overall performance of \method compare to that of other training sequence selection strategies for GR? 
    \item \textbf{Q2.} What efficiency benefits does \method provide?
    \item \textbf{Q3.} How well does \method generalize across data and models?
\end{itemize}




\subsection{Experimental Settings}
We describe the settings used for the experiments. 
See \cite{online2025appendix} for details.

\smallsection{Datasets.}
We evaluate our method on five public benchmarks and one industrial dataset.
The public datasets include \textbf{Beauty}, \textbf{Toys}, and \textbf{Sports}, which are 5-core-filtered subsets of the Amazon review dataset~\cite{mcauley2015image}.
\textbf{ML1M} and \textbf{ML20M} are widely used movie rating datasets from MovieLens \cite{harper2015movielens}.
\textbf{Internal} is a large-scale user sequence modeling dataset. 
User histories for training are collected over a 180-day period, and validation and test targets are the first events occurring during the subsequent week. We deduplicate consecutive events and remove users with fewer than 5 training events or fewer than 2 target events.
 Table~\ref{tab:data} summarizes dataset statistics.

\begin{table}[t!]
    \centering
    \caption{The statistics of the datasets.}
    \label{tab:data}
    \setlength\tabcolsep{2.7pt}
    \renewcommand{\arraystretch}{1.1}
    \scalebox{0.875}{
    \begin{tabular}{l|rrrrrr}
        \toprule
        \textbf{Dataset} & \textbf{Beauty} & \textbf{Toys} & \textbf{Sports} & \textbf{ML1M} & \textbf{ML20M} & \textbf{Internal} \\
        \midrule
        \# Users & 22,363 & 19,412 & 35,598 & 6,040 & 138,493 & 69.38M \\
        \# Items & 12,101 & 11,924 & 18,357 & 3,416 & 26,744 & 42,805\\
        \# Interactions & 198,502 & 167,597 & 296,337 & 999,611 & 20,000,263 & 1.527B \\
        \# Avg. Length & 8.88 & 8.63 & 8.32 & 165.50 & 144.41 & 22.01\\
        Sparsity & 99.93\% & 99.93\% & 99.95\% & 95.16\% & 99.46\% & 99.94\% \\
        \bottomrule
    \end{tabular}}
\end{table}

\begin{table*}[t!]
\centering
\caption{
\method consistently outperforms common augmentation strategies when applied to both SASRec~\cite{kang2018self} and TIGER~\cite{rajput2023recommender}, across datasets and metrics (NDCG@10 \& Recall@10).
The best performance is in \textbf{bold}, and the second-best one is \underline{\smash{underlined}}.
}
\label{tab:sasrec}
\setlength\tabcolsep{2.7pt}
\renewcommand{\arraystretch}{1.23}
\scalebox{0.80}{
\begin{tabular}{ll|cc|cc|cc|cc|cc}
\toprule
&
& \multicolumn{2}{c|}{\textbf{Beauty}} 
& \multicolumn{2}{c|}{\textbf{Toys}} 
& \multicolumn{2}{c|}{\textbf{Sports}} 
& \multicolumn{2}{c|}{\textbf{ML1M}} 
& \multicolumn{2}{c}{\textbf{ML20M}} \\
&& NDCG@10 & Recall@10 & NDCG@10 & Recall@10 & NDCG@10 & Recall@10 & NDCG@10 & Recall@10 & NDCG@10 & Recall@10 \\
\midrule

\multirow{5}{*}{SASRec} &

Last-Target
& 0.0124 \scriptsize$\pm$0.0005 & 0.0237 \scriptsize$\pm$0.0013 
& 0.0121 \scriptsize$\pm$0.0003 & 0.0237 \scriptsize$\pm$0.0008 
& 0.0037 \scriptsize$\pm$0.0003 & 0.0073 \scriptsize$\pm$0.0005 
& 0.0136 \scriptsize$\pm$0.0009 & 0.0306 \scriptsize$\pm$0.0016 
& 0.0628 \scriptsize$\pm$0.0010 & 0.1142 \scriptsize$\pm$0.0011 \\

& Multi-Target
& \underline{\smash{0.0372 \scriptsize$\pm$0.0006}} & \underline{\smash{0.0623 \scriptsize$\pm$0.0008}} 
& \underline{\smash{0.0378 \scriptsize$\pm$0.0014}} & \underline{\smash{0.0636 \scriptsize$\pm$0.0023}} 
& \underline{\smash{0.0162 \scriptsize$\pm$0.0006}} & \underline{\smash{0.0282 \scriptsize$\pm$0.0008}} 
& \underline{\smash{0.1194 \scriptsize$\pm$0.0046}} & \underline{\smash{0.2236 \scriptsize$\pm$0.0066}} 
& \underline{\smash{0.0995 \scriptsize$\pm$0.0015}} & \underline{\smash{0.1824 \scriptsize$\pm$0.0028}} \\

& Slide-Window
& 0.0323 \scriptsize$\pm$0.0004 & 0.0510 \scriptsize$\pm$0.0009 
& 0.0354 \scriptsize$\pm$0.0003 & 0.0571 \scriptsize$\pm$0.0006 
& 0.0149 \scriptsize$\pm$0.0001 & 0.0256 \scriptsize$\pm$0.0007 
& 0.1022 \scriptsize$\pm$0.0062 & 0.1960 \scriptsize$\pm$0.0082 
& 0.0526 \scriptsize$\pm$0.0024 & 0.1076 \scriptsize$\pm$0.0050 \\
\rowcolor{gray!10}
\cellcolor{white} & {\textbf{\method}}
& \textbf{0.0426 \scriptsize$\pm$0.0003} & \textbf{0.0689 \scriptsize$\pm$0.0007}
& \textbf{0.0481 \scriptsize$\pm$0.0007} & \textbf{0.0771 \scriptsize$\pm$0.0016}
& \textbf{0.0219 \scriptsize$\pm$0.0003} & \textbf{0.0365 \scriptsize$\pm$0.0003}
& \textbf{0.1230 \scriptsize$\pm$0.0029} & \textbf{0.2288 \scriptsize$\pm$0.0053}
& \textbf{0.1115 \scriptsize$\pm$0.0015} & \textbf{0.1938 \scriptsize$\pm$0.0024} \\
\rowcolor{gray!10}
\cellcolor{white} & {\textbf{Improv.}} & \textbf{14.52\%} & \textbf{10.59\%} & \textbf{27.25\%} & \textbf{21.23\%} & \textbf{35.19\%} & \textbf{29.43\%} & \textbf{3.02\%} & \textbf{2.33\%} & \textbf{12.06\%} & \textbf{6.25\%}\\

\bottomrule

\multirow{5}{*}{TIGER} &

Last-Target
& 0.0213 \scriptsize$\pm$0.0020 & 0.0431 \scriptsize$\pm$0.0035 
& 0.0212 \scriptsize$\pm$0.0002 & 0.0413 \scriptsize$\pm$0.0011 
& 0.0150 \scriptsize$\pm$0.0003 & 0.0281 \scriptsize$\pm$0.0009 
& 0.0147 \scriptsize$\pm$0.0012 & 0.0340 \scriptsize$\pm$0.0028 
& 0.0559 \scriptsize$\pm$0.0019 & 0.1063 \scriptsize$\pm$0.0028 \\

& Multi-Target
& 0.0319 \scriptsize$\pm$0.0003 & \underline{\smash{0.0608 \scriptsize$\pm$0.0007}} 
& \underline{\smash{0.0303 \scriptsize$\pm$0.0011}} & \underline{\smash{0.0575 \scriptsize$\pm$0.0019}} 
& \underline{\smash{0.0194 \scriptsize$\pm$0.0001}} & \underline{\smash{0.0359 \scriptsize$\pm$0.0002}} 
& \underline{\smash{0.1273 \scriptsize$\pm$0.0024}} & \underline{\smash{0.2272 \scriptsize$\pm$0.0050}} 
& \underline{\smash{0.1147 \scriptsize$\pm$0.0059}} & 
\underline{\smash{0.1900 \scriptsize$\pm$0.0091}} \\

& Slide-Window
& \underline{\smash{0.0321 \scriptsize$\pm$0.0014}} & 0.0580 \scriptsize$\pm$0.0009 
& 0.0273 \scriptsize$\pm$0.0007 & 0.0504 \scriptsize$\pm$0.0008 
& 0.0171 \scriptsize$\pm$0.0004 & 0.0319 \scriptsize$\pm$0.0010 
& 0.1105 \scriptsize$\pm$0.0038 & 0.1966 \scriptsize$\pm$0.0049 
& 0.0321 \scriptsize$\pm$0.0077 & 0.0646 \scriptsize$\pm$0.0147 \\

\rowcolor{gray!10}
\cellcolor{white} & {\textbf{\method}}
& \textbf{0.0443 \scriptsize$\pm$0.0010} & \textbf{0.0766 \scriptsize$\pm$0.0010}
& \textbf{0.0482 \scriptsize$\pm$0.0011} & \textbf{0.0822 \scriptsize$\pm$0.0034}
& {\textbf{0.0254 \scriptsize$\pm$0.0009}} & {\textbf{ 0.0453\scriptsize$\pm$0.0013}}
& {\textbf{0.1390 \scriptsize$\pm$0.0022}} & {\textbf{ 0.2425\scriptsize$\pm$0.0008}}
& \textbf{0.1233 \scriptsize$\pm$0.0003} & \textbf{0.2009 \scriptsize$\pm$0.0001} \\
\rowcolor{gray!10}
\cellcolor{white} & {\textbf{Improv.}} & \textbf{38.01\%} & \textbf{25.99\%} & \textbf{59.08\%} & \textbf{42.96\%} & \textbf{30.93\%} & \textbf{26.18\%} & \textbf{9.19\%} & \textbf{6.73\%} & \textbf{7.50\%} & \textbf{5.74\%}\\

\bottomrule
\end{tabular}
}
\end{table*}

\smallsection{Evaluation.}
We use NDCG@$K$ and Recall@$K$ to evaluate all methods, with $K \in \{5, 10\}$.
We follow the leave-one-out protocol to split the data into training, validation, and test sets~\cite{kang2018self,rajput2023recommender,xie2022contrastive,zhou2020s3}.
The last and second-to-last items are used for testing and validation, respectively, and the remainder form the training set.

\smallsection{Implementation Details.}
We implement SASRec~\cite{kang2018self} using a contrastive loss with full-batch negative samples. 
We use embedding dimension 128 and two transformer layers with one attention head per layer.
Following \cite{zhao2022recbole}, we compute the loss using one target per sequence in the batch. 
All sampling strategies produce the same number of subsequences per batch, each with a single training target.
For TIGER~\cite{rajput2023recommender}, we use the publicly-released implementation~\cite{ju2025generative}.
Semantic IDs are obtained by extracting 4096-dimensional text embeddings using Flan-T5-XXL~\cite{chung2024scaling}, followed by residual k-means clustering~\cite{deng2025onerec} with three layers and 256 codebooks per layer.
For the sequential model, we use the default T5-based encoder-decoder transformer.
We train an internal SASRec-style model for the \textbf{Internal} experiments.
We use the AdamW optimizer and tune the learning rate and weight decay separately for each model. 
We use batch sizes of 2096 for public data and 256 for \textbf{Internal} data, and early stop all experiments using validation Recall@5.

For computational efficiency, we implement \method by sampling subsequences in-batch. 
That is, starting from a set of $B$ raw sequences, we sample $B$ subsequences for all methods, where $B$ is the batch size.
The raw $(\alpha,\beta,\gamma)$ configuration set $\mathcal{S}$ (see Section~\ref{sec:method:parameter_search}) is given by $\alpha \in \{-2,-1,0,1,2\}$, $\beta \in \{-1,0,1,2,\infty\}$, and $\gamma \in \{-\infty,-1,0,1\}$, encompassing the three strategies described in Section~\ref{sec:prelim:strategies}. We set $r=20$ and $k=10$ for public data, and $k=5$ for \textbf{Internal}.
On \textbf{Internal}, we sample 150K users to estimate dataset statistics, and restrict $\gamma$ to $\{-\infty,0\}$, resulting in $(\alpha, \beta, \gamma)=(0,2,0)$.

\smallsection{Machines.}
We run the academic and internal experiments on 4 x 16GB T4 and 8 x 40GB A100 NVIDIA GPUs, respectively.

\vspace{-2mm}

\subsection{Q1.  Overall Performance}\label{sec:exp:effectiveness}
First and most importantly, we assess the \textit{effectiveness} of \method.

\smallsection{Overall Performance.}
In Table~\ref{tab:sasrec}, we compare \method with common data augmentation strategies across five public datasets.
\method consistently and significantly outperforms all baselines, with improvements over the best-performing baseline ranging from 2.33\% to 59.08\%.
These results indicate that commonly used strategies do not yield an optimal training distribution, whereas \method effectively reshapes it for better performance.
Additional results for $K$=5 are provided in \cite{online2025appendix}.

\smallsection{On Training Data Properties.}
We analyze how the augmented training data produced by \method, using the selected $(\alpha,\beta,\gamma)$ configurations, reshapes the training distribution.
As shown in Table~\ref{tab:alpha_beta_gamma}, these configurations generally yield improved target distributions and input–target distributions compared to MT-augmented data.
Notably, none of the common strategies (LT, MT, or SW) are selected as optimal, highlighting the importance of flexibly and appropriately shaping the training distribution for effective GR.

\smallsection{Comparison with Sequence-Level Augmentations.}
We compare \method with common sequence-level augmentation methods~\cite{zhou2024contrastive} (also see Appendix~\ref{app:input_aug} for details) under various target distributions from different $(\alpha,\beta)$ settings. 
As shown in Table~\ref{tab:input_aug}, \method outperforms these perturbation-based baselines (e.g., item insertion and reordering) when $\gamma$ is properly tuned, by sampling \textit{contiguous} inputs of varying length. 
While prior work often evaluates only at $(\alpha,\beta)=(0,\infty)$~\cite{zhou2024contrastive}, our results show performance varies with the target distribution, emphasizing the need for broader evaluation.

\begin{figure}[t]
    \centering
    \includegraphics[width=0.48\linewidth]{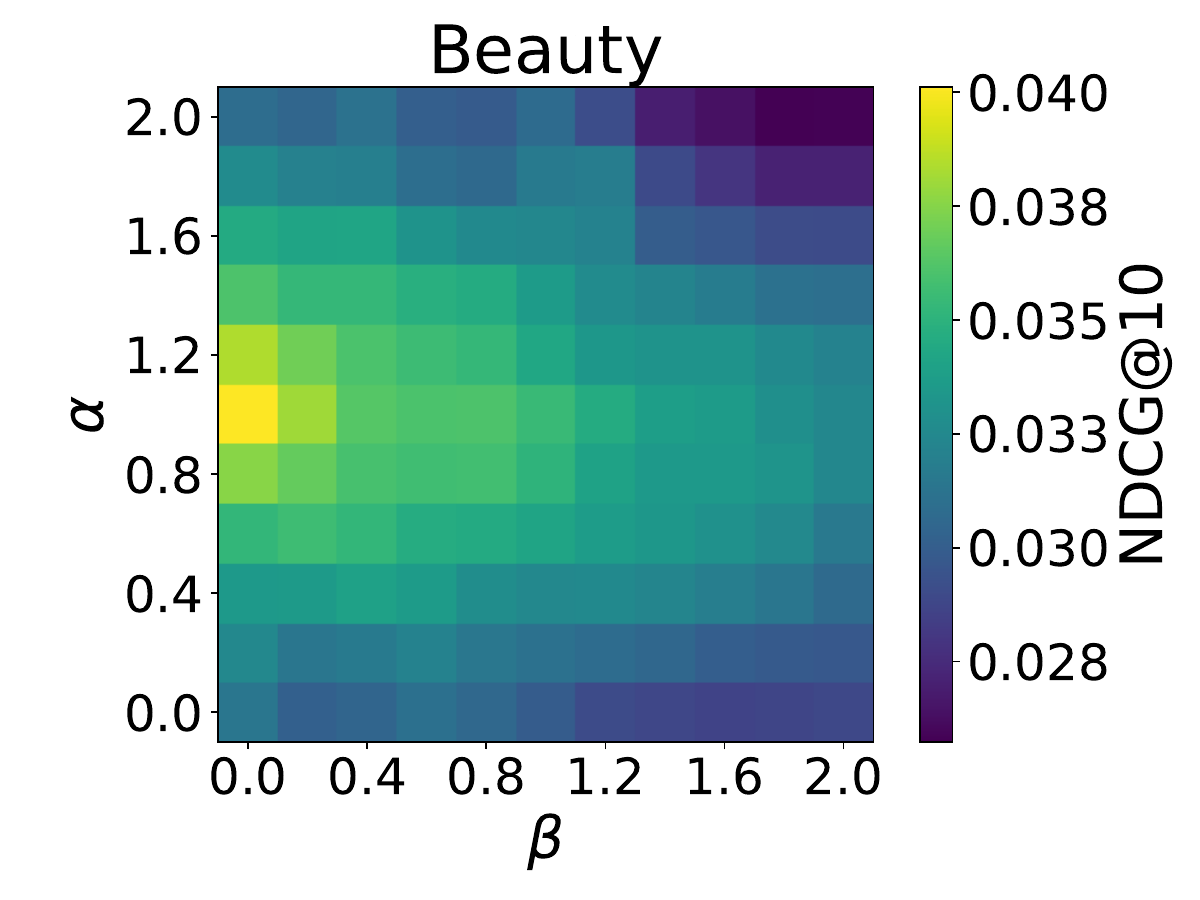}
    \hspace{5pt}
    \includegraphics[width=0.48\linewidth]{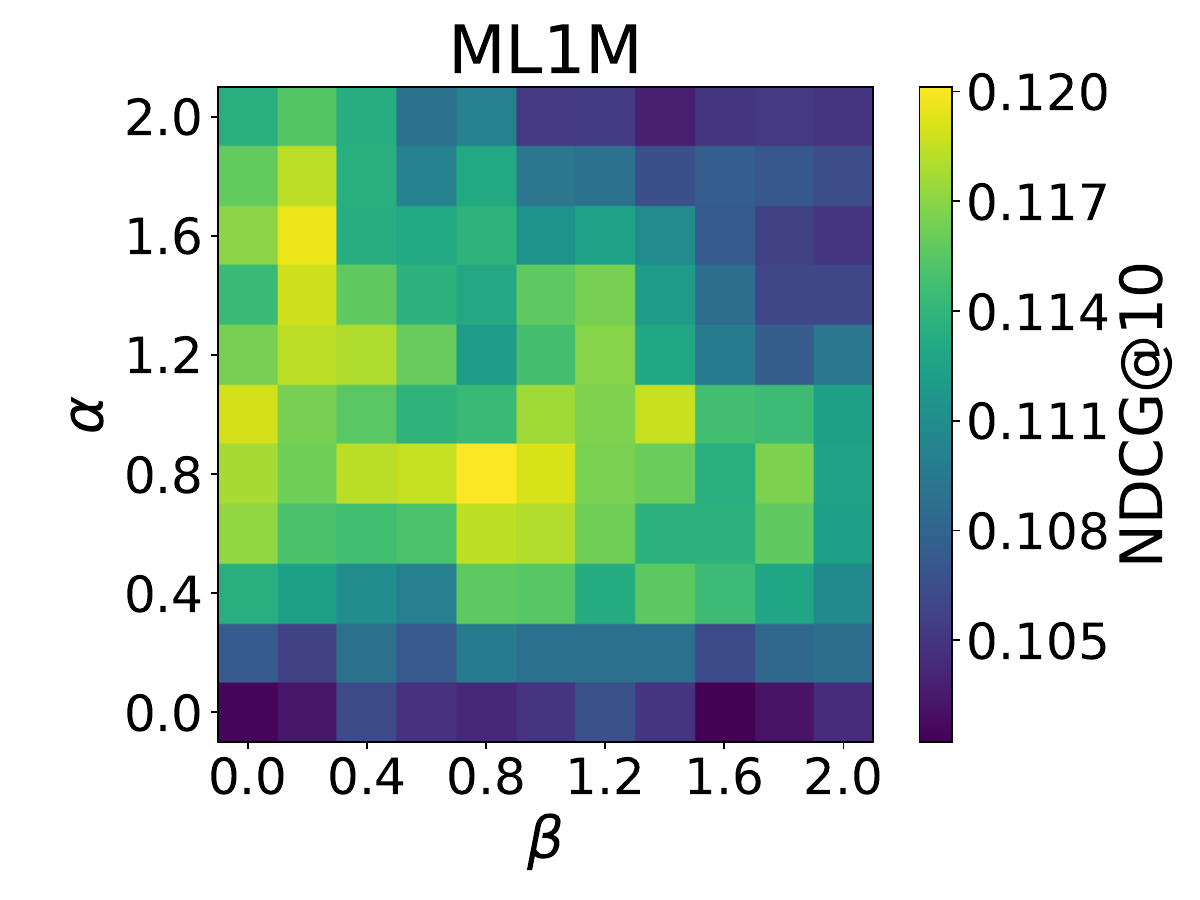}
    \caption{
    The two parameters, $\alpha$ and $\beta$, jointly shape the training distribution and have a substantial impact on model performance.
    Their impact patterns differ across datasets.
    \label{fig:alpha_beta}}
\end{figure}

\smallsection{Effectiveness of Components.}
We examine the impact of \method’s key parameters. Figure~\ref{fig:alpha_beta} shows that varying $\alpha$ (user bias) and $\beta$ (target bias) can significantly affect performance, while Table~\ref{tab:input_aug} shows that tuning $\gamma$ under different $(\alpha,\beta)$ settings further influences results. 
Both trends and sensitivities vary across datasets, emphasizing the need for flexible parameter control to shape the training distribution for effective generative recommendation.

\smallsection{Validation of Parameter Search.}
We evaluate the effectiveness of our parameter search method (Section~\ref{sec:method:parameter_search}).
As shown in Figure~\ref{fig:param}, SASRec performs better when the KL divergence between training and test target distributions is low and when the alignment-to-discrimination ratio across training and test input–target distributions is high.
This supports our parameter search strategy that retains a subset of $(\alpha,\beta,\gamma)$ configurations with low target KL divergence and high alignment-to-discrimination ratios.

\subsection{Q2. Efficiency }\label{sec:exp:efficiency}
We evaluate the \textit{efficiency} of model training with \method in three aspects: search, parameter, and data efficiency.


\begin{table}[t!]
    \centering
    \caption{
    Selected $(\alpha,\beta,\gamma)$ configurations of \method from the reduced search space for each dataset and model.
    \colorbox{green!10}{\textcolor{black}{Green}} indicates training data that improves both the target distribution KL divergence and the alignment–discrimination ratio of the input–target distribution compared to MT (see Table~\ref{tab:kl_align_disc}).
    \colorbox{red!10}{\textcolor{black}{Red}} indicates degradation.
    \colorbox{gray!10}{\textcolor{black}{Gray}} indicates no change.
    }
    \label{tab:alpha_beta_gamma}
    \setlength\tabcolsep{4.75pt}
    \renewcommand{\arraystretch}{1.1}
    \scalebox{0.89}{
    \begin{tabular}{lc|ccccc}
        \toprule
         && \textbf{Beauty} & \textbf{Toys} & \textbf{Sports} & \textbf{ML1M} & \textbf{ML20M} \\
        \midrule
        \multirow{3}{*}{SASRec} &
        $(\alpha,\beta,\gamma)$ & (1, 0, 0) & (1, 0, 0) & (1, 0, 0) & (1, 0, 0) & (0, 2, -$\infty$)  \\
        & KL ($\downarrow$) & \cellcolor{gray!10}0.898 & \cellcolor{gray!10}1.062 & \cellcolor{gray!10}0.819 & \cellcolor{gray!10}0.495 & \cellcolor{green!10}0.158 \\
        & A/D ($\uparrow$) & \cellcolor{green!10}17.95 & \cellcolor{green!10}24.57 & \cellcolor{green!10}11.09 & \cellcolor{red!10}3.68 & \cellcolor{green!10}12.98 \\
        \midrule
        \multirow{3}{*}{TIGER} & 
        $(\alpha,\beta,\gamma)$ & (1, 0, 1) & (1, 0, 1) & (1, 1, 1) & (1, 0, 1) & (0, 2, -$\infty$) \\
        & KL ($\downarrow$) & \cellcolor{gray!10}0.898 & \cellcolor{gray!10}1.062 & \cellcolor{green!10}0.768 & \cellcolor{gray!10}0.495 & \cellcolor{green!10}0.158 \\
        & A/D ($\uparrow$)  & \cellcolor{green!10}18.80 & \cellcolor{green!10}25.54 & \cellcolor{green!10}11.11 & \cellcolor{red!10}3.86 & \cellcolor{green!10}12.98 \\
        \bottomrule
    \end{tabular}}
\end{table}

\begin{figure}[t!]
    \centering
    \includegraphics[width=0.495\linewidth]{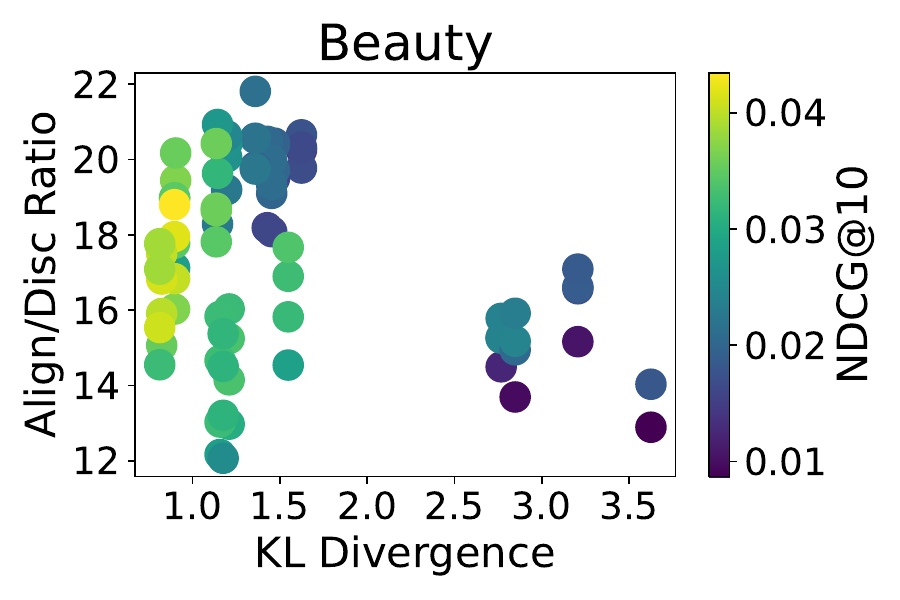}
    \includegraphics[width=0.495\linewidth]{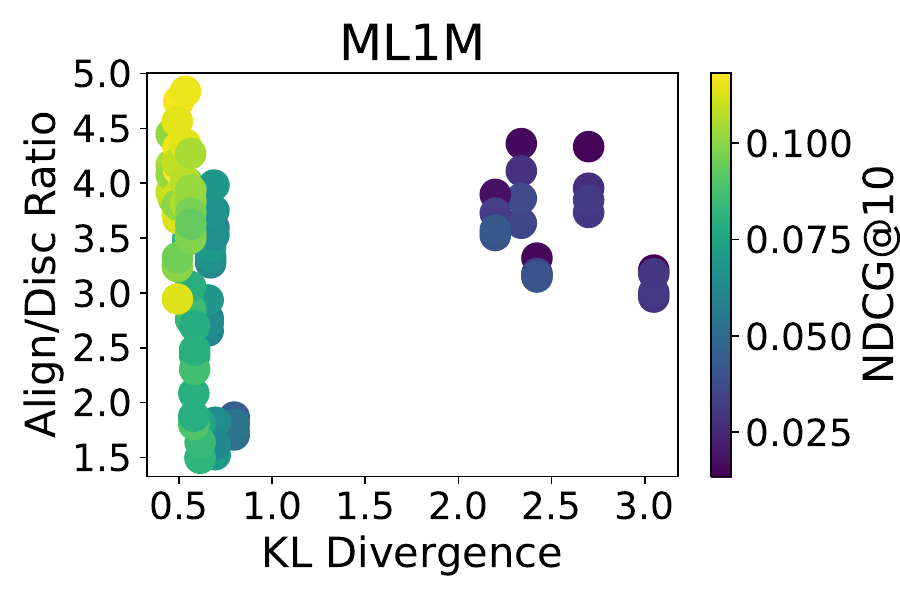}
    \caption{
    The performance (NDCG@10) of SASRec tends to improve when the KL divergence between the training and test target distributions is low, and when the ratio of alignment to discrimination in the input-target distribution between training and test sets is high.
    }
    \label{fig:param}
\end{figure}

\begin{table*}[t!]
    \centering
    \caption{
    Comparison of sequence-level augmentation methods and \method across different $(\alpha,\beta)$ configurations, in terms of NDCG@10.
    While baseline strategies perturb input sequences, \method samples contiguous subsequences of varying lengths by flexibly controlling $\gamma$.
    Across diverse settings, \method achieves superior performance when $\gamma$ is properly chosen.
    \label{tab:input_aug}
    }
    \setlength\tabcolsep{5.7pt}
    \renewcommand{\arraystretch}{1.07}
    \scalebox{0.885}{
    \begin{tabular}{l|ccc|ccc|ccc|ccc|ccc}
        \toprule
        & \multicolumn{3}{c|}{\textbf{Beauty}} & \multicolumn{3}{c|}{\textbf{Toys}} & \multicolumn{3}{c|}{\textbf{Sports}} & \multicolumn{3}{c|}{\textbf{ML1M}} & \multicolumn{3}{c}{\textbf{ML20M}} \\
        $(\alpha,\beta)\rightarrow$ & $(0,\infty)$ & $(1,0)$ & $(2,1)$ & $(0,\infty)$ & $(1,0)$ & $(2,1)$ & $(0,\infty)$ & $(1,0)$ & $(2,1)$ & $(0,\infty)$ & $(1,0)$ & $(2,1)$ & $(0,\infty)$ & $(1,0)$ & $(2,1)$ \\
        \midrule
        Insert & 
        0.0192 & 0.0399 & 0.0310 & 
        0.0217 & 0.0446 & 0.0322 & 
        0.0080 & 0.0203 & 0.0136 & 
        0.0182 & 0.1179 & 0.0984 &
        0.0637 & 0.0973 & 0.0679
        \\
        Delete & 
        0.0147 & 0.0324 & 0.0233 & 
        0.0152 & 0.0349 & 0.0227 & 
        0.0052 & 0.0156 & 0.0092 & 
        0.0153 & 0.0711 & 0.0713 &
        0.0531 & 0.0619 & 0.0462
        \\
        Replace & 
        0.0148 & 0.0308 & 0.0243 & 
        0.0149 & 0.0284 & 0.0183 & 
        0.0051 & 0.0130 & 0.0066 & 
        0.0176 & 0.1122 & 0.0971 &
        0.0605 & 0.0893 & 0.0652
        \\
        Reorder & 
        0.0129 & 0.0353 & 0.0267 & 
        0.0128 & 0.0353 & 0.0225 & 
        0.0041 & 0.0154 & 0.0092 & 
        0.0147 & \underline{\smash{0.1202}} & 0.1074 &
        0.0624 & 0.0969 & 0.0701
        \\
        Sample & 
        0.0159 & 0.0376 & 0.0299 & 
        0.0170 & 0.0383 & 0.0270 & 
        0.0060 & 0.0171 & 0.0110 & 
        0.0194 & 0.1143 & 0.0989 &
        0.0601 & 0.0911 & 0.0649
        \\
        \midrule
        \rowcolor{gray!10}
        $\gamma=-\infty$ & 
        0.0124 & 0.0372 & 0.0274 & 
        0.0121 & 0.0378 & 0.0226 & 
        0.0037 & 0.0162 & 0.0091 &
        0.0136 & 0.1194 & 0.1100 &
        0.0628 & \textbf{0.0995} & 0.0564
        \\
        \rowcolor{gray!10}
        $\gamma= -1$ & 
        0.0221 & 0.0403 & 0.0323 & 
        0.0267 & 0.0455 & 0.0334 &  
        0.0099 & 0.0195 & 0.0133 &  
        0.0323 & 0.1107 & 0.0993 &
        0.0693 & \underline{\smash{0.0973}} & 0.0667 
        \\
        \rowcolor{gray!10}
        $\gamma =0$ &
        0.0236 & \textbf{0.0426} & 0.0323 & 
        0.0287 & \textbf{0.0481} & 0.0354 & 
        0.0109 & \textbf{0.0219} & 0.0149 & 
        0.0382 & \textbf{0.1230} & 0.1022 & 
        0.0749 & 0.0937 & 0.0526
        \\
        \rowcolor{gray!10}
        $\gamma=1$ & 
        0.0240 & \underline{\smash{0.0410}} & 0.0323 & 
        0.0287 & \underline{\smash{0.0473}} & 0.0354 &  
        0.0112 & \underline{\smash{0.0216}} & 0.0144 &  
        0.0420 & 0.1041 & 0.0977 &
        0.0759 & 0.0869 & 0.0625 
        \\
        \bottomrule
    \end{tabular}}
\end{table*}

\smallsection{Search Efficiency.}
We examine how our parameter search scheme (Section~\ref{sec:method:parameter_search}) reduces model tuning time. 
For example, on Beauty, training SASRec and TIGER takes  roughly 1 and 9 hours, respectively, making exhaustive tuning over  $\mathcal{S}$ expensive. 
Instead, we retain only $k / |\mathcal{S}|$ ($10/100$ for our public data experiments) of configurations. 
Computing KL divergence takes 0.07 seconds per configuration, and alignment and discrimination take 52 seconds, both negligible compared to training time.

\smallsection{Parameter Efficiency.}\label{app:details_experiments:large_model}
In Figure~\ref{fig:large}, we compare small models using \method with larger ones with other strategies.
Our default SASRec uses (embedding dimension, \# attention heads, \# layers) = (128, 1, 2).
For large models, we scale capacity to: (256, 1, 2), (256, 2, 2), (512, 2, 2), (512, 2, 4), (512, 2, 8), (512, 4, 2), (512, 4, 4), and (512, 4, 8).
Despite significantly fewer parameters, small models trained with \method consistently outperform much larger counterparts, demonstrating strong parameter efficiency.

\smallsection{Data Efficiency.}
To assess data efficiency, we train models with \method on reduced training subsets and compare them to full-data without augmentation.
As shown in Figure~\ref{fig:small_data}, the model equipped with \method achieves strong performance even with minimal data.
For example, it outperforms the full-data baseline using only 1\% of the data on ML1M.
This demonstrates that \method effectively leverages available data to improve learning efficiency.

\subsection{Q3. Generalizability}\label{sec:exp:generalizability}
We evaluate \method's \textit{generalizability}, whether it consistently improves performance across diverse and practical scenarios.

\smallsection{Long-Tail Performance.}
To assess how well \method generalizes to infrequent items, we divide items into three equal-sized groups (G1–G3), where G1 contains the least popular and G3 the most popular items in the training set. 
As shown in Figure~\ref{fig:long_tail}, models trained on \method-augmented data outperform those trained with other augmentation strategies across all groups, demonstrating \method’s improved generalization to less popular items.

\begin{figure}[t!]
    \centering
    \includegraphics[width=0.215\linewidth]{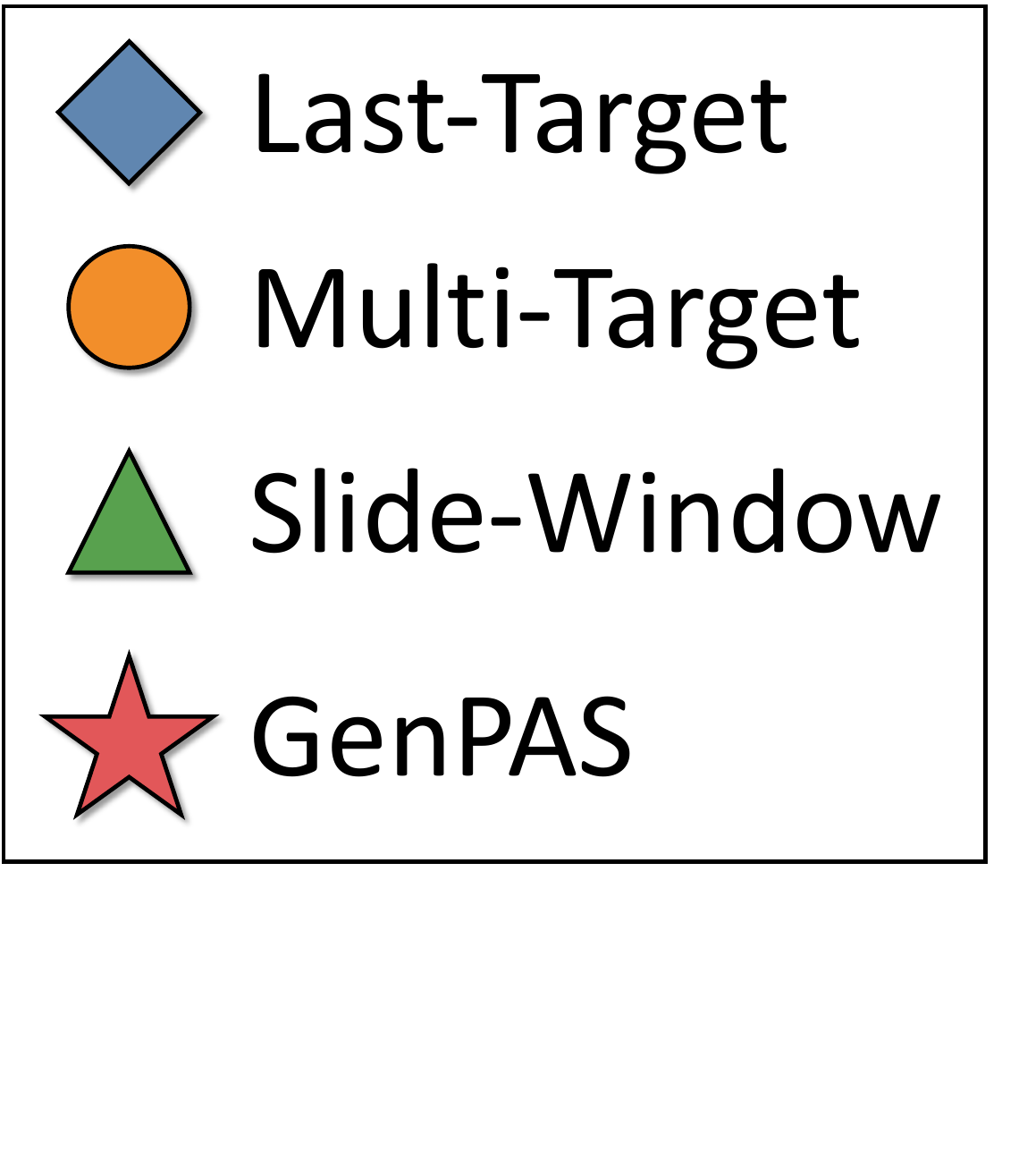}
    \hspace{12pt}
    \includegraphics[width=0.47\linewidth]{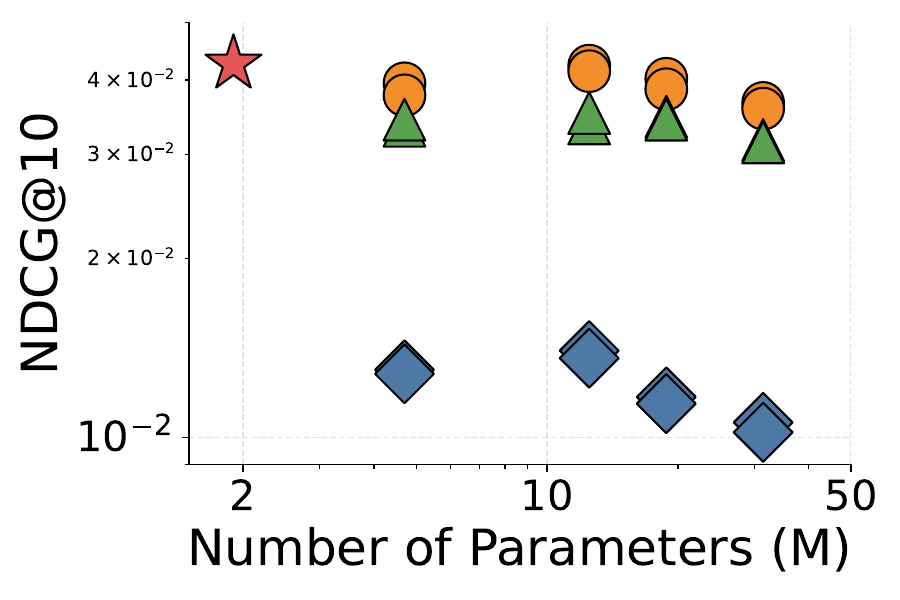}
    \caption{
    SASRec trained on \method-augmented data outperforms larger-parameter variants trained with other strategies, demonstrating \method's parameter efficiency.
    }
    \label{fig:large}
\end{figure}



\smallsection{Large-scale Data.}
We study how  \method generalizes to industrial-scale data using the \textbf{Internal} dataset in
Table \ref{tab:Internal}. Here, we evaluate on the test targets from a random subset of 200k training users (transductive setting) and a set of 200k users unseen during training (inductive setting). We randomly sample 2000 negative items for evaluation in both  settings, and compare against the internal baseline that uses LT training. Even on this large-scale data, for which the default LT dataset already has many training samples, augmenting via GenPAS  confers large performance improvements.





\vspace{-2mm}

%% file: 050related.tex
In this section, we review prior work relevant to our study.

\smallsection{Generative Recommendation (GR).}
Modeling sequential user behavior is crucial for capturing user preferences.
Early methods relied on the Markov Chain assumptions, predicting the next interaction from recent ones~\cite{he2016fusing,rendle2010factorizing}.
With deep learning, more expressive models emerged, including GRUs~\cite{hidasi2015session}, RNNs~\cite{li2017neural,wang2020intention}, CNNs~\cite{jiang2023adamct,tang2018personalized}, and GNNs~\cite{chang2021sequential,wang2020modelling}, to capture complex dependencies.
Attention-based models soon became dominant in sequential recommendation~\cite{kang2018self,sun2019bert4rec,li2020time}, selectively focusing on relevant past interactions.
Recently, transformer-based ``generative'' models have been increasingly adopted for sequential recommendation~\cite{lin2025can,zhao2024recommender,rajput2023recommender,jin2023language,petrov2023generative,chen2024enhancing}, following successes in NLP~\cite{vaswani2017attention,bai2023qwen,brown2020language} and computer vision~\cite{dosovitskiy2020image,he2022masked}.  
Following this paradigm, several works explore semantic identifiers in generative recommendation~\cite{rajput2023recommender, deng2025onerec, lee2025gram, grid}.

\begin{figure}[t]
    \centering
    \includegraphics[width=0.8\linewidth]{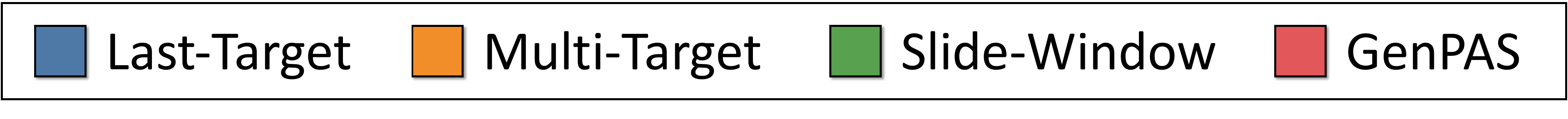}\\
    \includegraphics[width=0.4525\linewidth]{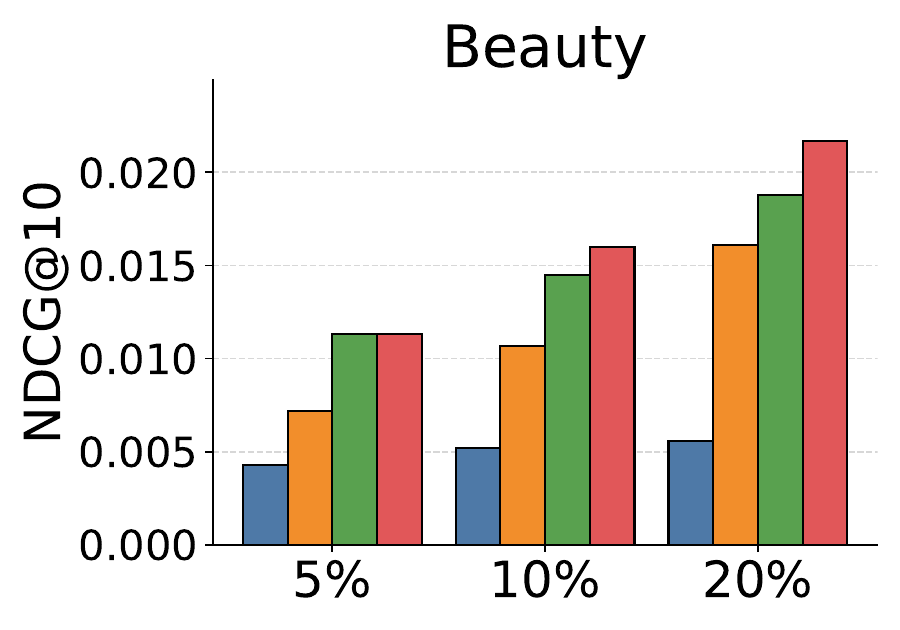}
    \hspace{3pt}
    \includegraphics[width=0.4525\linewidth]{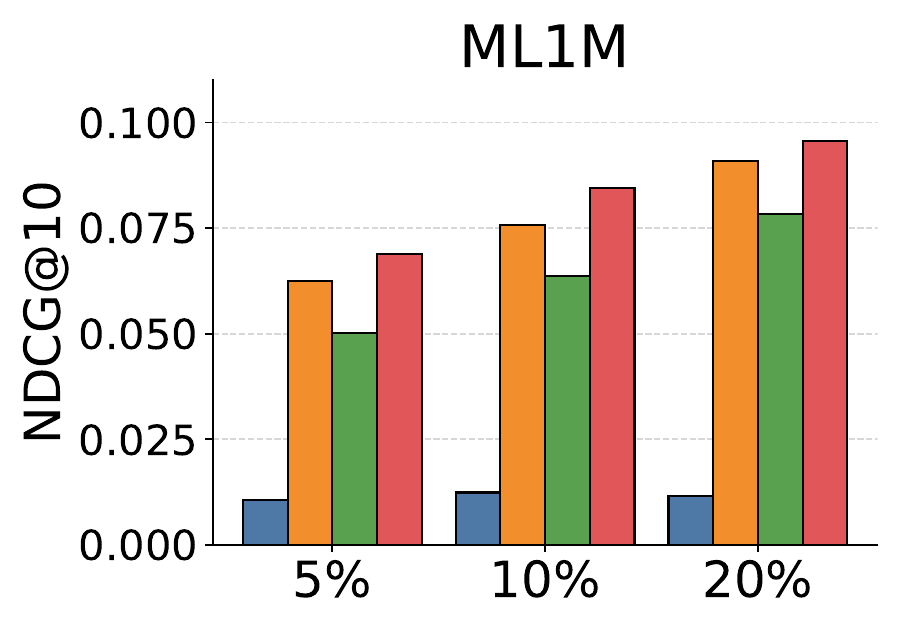}
    \caption{
    \method enhances the data efficiency.
    SASRec with \method outperforms the full-data baseline without augmentation, even when trained on 5, 10, 20\% of the original data. 
    \label{fig:small_data}}
\end{figure}

\begin{figure}[t]
    \centering
    \includegraphics[width=0.8\linewidth]{FIG_NEW/legend_all.pdf}\\
    \includegraphics[width=0.4525\linewidth]{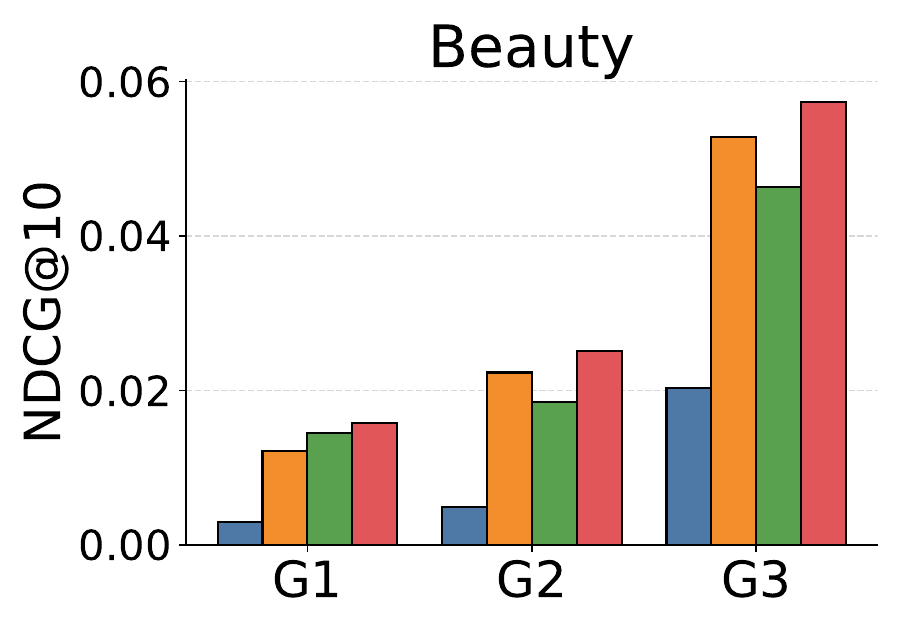}
    \hspace{-3pt}
    \includegraphics[width=0.4525\linewidth]{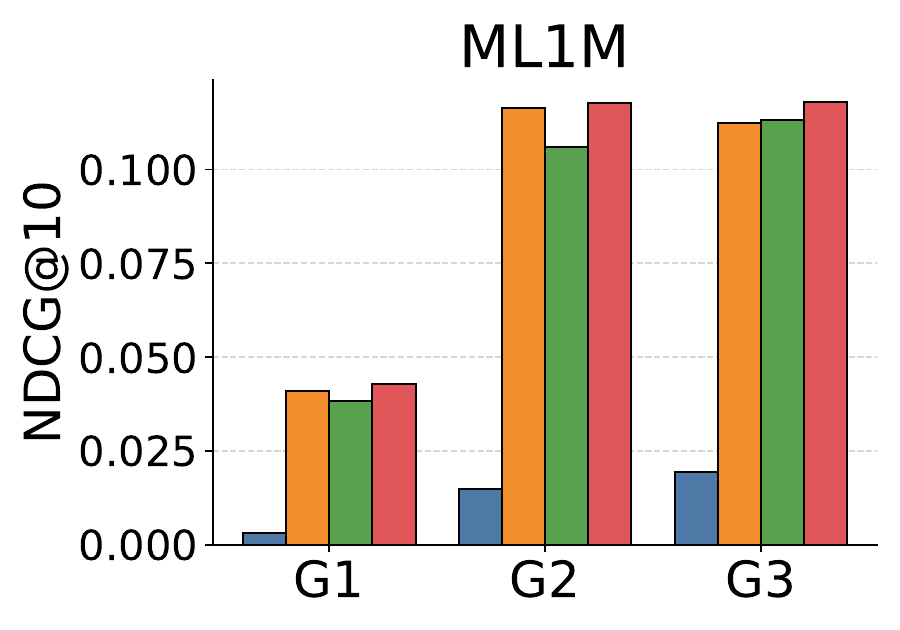}
    \caption{
    \method enhances long-tail performance. 
    SASRec equipped with \method consistently outperforms the non-augmented model across all item groups, from the least popular (G1) to the most popular (G3).
    \label{fig:long_tail}}
\end{figure}

\smallsection{Data Augmentation for GR.}
Data augmentation in GR refers to shaping the training data by controlling the distribution of input-target pairs. 
While the Last-Target strategy predicts only the final item~\cite{zhou2024contrastive}, SASRec~\cite{kang2018self} and many follow-up studies use the Multi-Target strategy that includes intermediate items as prediction targets~\cite{zhou2020s3,chen2022intent}. 
A recent reproducibility study~\cite{grid} showed that reliably reproducing TIGER~\cite{rajput2023recommender} requires the Slide-Window strategy, which generates input–target pairs from all subsequence windows.
These findings suggest that the choice of prediction targets within a sequence can substantially influence both model training and reported performance.
Beyond target selection, GR explores various input transformation strategies, including item insertion, deletion, replacing, reordering, and sampling~\cite{online2025appendix,zhou2024contrastive}.
More advanced approaches include counterfactual augmentation~\cite{wang2021counterfactual,zhang2021causerec} and test-time augmentation~\cite{dang2025data} (see surveys~\cite{dang2024data,song2022data}).
While most prior work employs augmentation as a tool for contrastive learning~\cite{liu2021contrastive,xie2022contrastive,chen2022intent}, recent findings suggest that augmentation itself can be highly effective for GR~\cite{xie2022contrastive}. Nevertheless, our work is the first to thoroughly investigate sequence selection in GR.




%% file: 060conc.tex
In this work, we revisited data augmentation for generative recommendation and revealed that widely used strategies, though simple and popular, can be far from optimal.
Through systematic analysis, we showed how augmentation reshapes both the target and input–target distributions, directly influencing alignment with future targets and generalization to unseen inputs.
Based on our observations, we have developed a novel, effective method for tuning the strategy. Our analysis is data-centric; future work remains to study how  GR model architecture and training objectives influence the optimal sequence selection strategy.



%% file: 070appendix.tex
\begin{table}[t!]
    \centering
    \caption{
    The performance of \method  on large-scale data.
    \label{tab:Internal}
    }
    \setlength\tabcolsep{3pt}
    \renewcommand{\arraystretch}{1.1}
    \scalebox{0.83}{
   \begin{tabular}{c|cc|cc|cc|cc}
        \toprule
        & \multicolumn{4}{c|}{\textbf{Transductive}} & \multicolumn{4}{c}{\textbf{Inductive (New Users)}} \\
        & \multicolumn{2}{c|}{Baseline (LT)} & \multicolumn{2}{c|}{\method} & \multicolumn{2}{c|}{Baseline (LT)} & \multicolumn{2}{c}{\method} \\
        & { N@10} & { R@10} & { N@10} & {R@10} & { N@10} & { R@10} & { N@10} & { R@10} \\
        \midrule
        & 0.1904 & 0.3144 & \textbf{0.2059} & \textbf{0.3512} & 0.1961  &  0.3144 & \textbf{0.2104} & \textbf{0.3357} \\
        Improv. & -- & -- & \textbf{+8.14\%} & \textbf{+11.7\%} & --  &  -- & \textbf{+7.29\%} & \textbf{+6.77\%} \\
        \bottomrule
    \end{tabular}}
\end{table}

\section{Generalization and Proof of Theorem \ref{thm:targets}} \label{sec:theory}
Theorem \ref{thm:targets} follows from the below result by setting $\beta\in \{0,1,2\}$.
\begin{theorem} \label{thm:targets-gen}
    Suppose Assumptions \ref{assump:indep_users} and \ref{assump:indep_items} hold. Denote $\delta_k := \text{TV}(p_k, p_{n+1})$ and $m:= |\mathcal{U}|$, and let ${p}^{\beta}_{\text{train}}$ be the training target distribution induced via \method with the corresponding $\beta$.  Then, with probability at least $0.99$,
    we have 
    \begin{align*}
        \text{TV}({p}_{\text{train}}^{\beta}, p_{n+1}) &= 
        O\left(\frac{\sum_{k=1}^n k^\beta \delta_k + |\mathcal{I}|\sqrt{\frac{\log(|\mathcal{I}|)}{m}\sum_{k=1}^n k^{2\beta} } }{\sum_{k=1}^n k^{\beta}} \right).
    \end{align*}
\end{theorem}

\begin{proof}
Please see the Online Appendix~\cite{online2025appendix}.
\end{proof}







\section{Sequence-Level Augmentation Strategies}\label{app:input_aug}
We review sequence-level strategies for GR~\cite{zhou2024contrastive} tested in Section~\ref{sec:exp:effectiveness}. 
All take input $s^{(u)}$.

\noindent\textbf{Insert.}
Inserts a randomly sampled item $i^*\in \mathcal{I}$ at a random position $k^* \in \{1,\dots,|s^{(u)}|-1\}$:
\begin{equation*}\label{eq:insert}\small
    \tilde{x}^{(u)} = \left[ i_1^{(u)},\dots,i_{k^*-1}^{(u)},\;i^* \;,i_{k^*}^{(u)},\dots,i_{|s^{(u)}|-1}^{(u)} \right].
\end{equation*}

\noindent\textbf{Delete.}
Removes the item at a randomly selected position $k^*\in\{1,\dots,|s^{(u)}|-1\}$:
\begin{equation*}\label{eq:delete}\small
    \tilde{x}^{(u)} = \left[ i_1^{(u)},\dots,i_{k^*-1}^{(u)},i_{k^*}^{(u)},\dots,i_{|s^{(u)}|-1}^{(u)} \right].
\end{equation*}

\noindent\textbf{Replace.}
Replaces the item at a randomly selected position $k^*\in\{1,\dots,|s^{(u)}|-1\}$ with a randomly sampled item $i^*\in\mathcal{I}$:
\begin{equation*}\label{eq:replace}\small
\tilde{x}^{(u)} = \left[ i_1^{(u)}, \dots, i_{k^*-1}^{(u)},\; i^*,\; i_{k^*}^{(u)}, \dots, i_{|s^{(u)}|-1}^{(u)} \right].
\end{equation*}

\noindent\textbf{Reorder.}
Shuffles a contiguous subsequence of length $\delta$.
Let $k^*\in\{1,\dots,|s^{(u)}|-\delta\}$ be the start index, and let $\texttt{shuffle}(\cdot)$ denote a random permutation function.
The augmented input sequence is:
\begin{equation*}\label{eq:reorder}\small
\tilde{x}^{(u)} = \left[ i_1^{(u)}, \dots, i_{k^* - 1}^{(u)}, {\small\texttt{shuffle}}(i_{k^*}^{(u)}, \dots, i_{k^* + \delta - 1}^{(u)}), i_{k^* + \delta}^{(u)}, \dots, i_{|s^{(u)}|-1}^{(u)} \right].
\end{equation*}

\noindent\textbf{Sample.}
Samples each item with a retention probability $\omega\in(0,1)$:
\begin{equation*}\label{eq:sample}\small
\tilde{x}^{(u)} = \left[ \widehat{i}_{1}^{(u)}, \dots, \widehat{i}_{|s^{(u)}|-1}^{(u)} \right] \;\;\text{where}\;\;\; \widehat{i}_k^{(u)} = \begin{cases}
i_k^{(u)}, & \text{with prob.}\; \omega \\
\varnothing, & \text{with prob.}\; 1 - \omega
\end{cases}
\end{equation*}
for $k=1,\dots,|s^{(u)}|-1$.



%% file: 080ethic.tex
\method can influence item exposure in generative recommendation, potentially exacerbating popularity bias if misused. 
To mitigate this, we suggest properly tuning parameters $\alpha$, $\beta$, and $\gamma$ to maintain balanced target distributions and encourage diverse inputs.
